\documentclass{article}

\usepackage{microtype}
\usepackage{graphicx}
\usepackage{bbm}
\usepackage{multirow}
\usepackage{booktabs} % for professional tables

\usepackage{hyperref}

\usepackage{amsmath}
\usepackage{amssymb}
\usepackage{mathtools}
\usepackage{amsthm,thmtools}
\usepackage{todonotes}

\usepackage[capitalize,noabbrev]{cleveref}

\theoremstyle{plain}
\newtheorem{theorem}{Theorem}[section]
\newtheorem{proposition}[theorem]{Proposition}
\newtheorem{lemma}[theorem]{Lemma}

\theoremstyle{definition}

\newtheorem{assumption}[theorem]{Assumption}
\theoremstyle{remark}
\newtheorem{remark}[theorem]{Remark}

\usepackage[preprint]{neurips_2024}

\usepackage[utf8]{inputenc} % allow utf-8 input
\usepackage[T1]{fontenc}    % use 8-bit T1 fonts
\usepackage{hyperref}       % hyperlinks
\usepackage{url}            % simple URL typesetting
\usepackage{booktabs}       % professional-quality tables
\usepackage{amsfonts}       % blackboard math symbols
\usepackage{nicefrac}       % compact symbols for 1/2, etc.
\usepackage{microtype}      % microtypography
\usepackage{xcolor}         % colors

\usepackage{subcaption}
\usepackage{pifont} % for the cross symbol

\title{From Conformal Predictions to Confidence Regions}

\author{%
  Charles Guille-Escuret \\
  Mila, Apple
  \And
  Eugene Ndiaye \\
  Apple \\
}

\begin{document}

\maketitle

\begin{abstract}
Conformal prediction methodologies have significantly advanced the quantification of uncertainties in predictive models. Yet, the construction of confidence regions for model parameters presents a notable challenge, often necessitating stringent assumptions regarding data distribution or merely providing asymptotic guarantees. We introduce a novel approach termed \texttt{CCR}, which employs a combination of conformal prediction intervals for the model outputs to establish confidence regions for model parameters. We present coverage guarantees under minimal assumptions on noise and that is valid in finite sample regime. Our approach is applicable to both split conformal predictions and black-box methodologies including full or cross-conformal approaches. In the specific case of linear models, the derived confidence region manifests as the feasible set of a Mixed-Integer Linear Program (MILP), facilitating the deduction of confidence intervals for individual parameters and enabling robust optimization. We empirically compare \texttt{CCR} to recent advancements in challenging settings such as with heteroskedastic and non-Gaussian noise.
\end{abstract}

\section{Introduction}

To ensure the safe deployment of machine learning technologies, it is essential to have reliable methods for quantifying uncertainties when making data-driven predictions.  Bayesian methods \citep{bernardo2009bayesian}, for example, rely on posterior distribution to model these uncertainties. The process begins with a prior distribution that represents prior knowledge before data are observed. Bayes' rule is then used to obtain a data-dependent probability called the posterior distribution, which represents the updated belief conditional on the data. The resulting distribution is commonly used to measure the likelihood of a specific outcome. However, a good choice of the prior is required in order to have reliable uncertainty estimates, which is difficult without stronger assumptions; see \citep{castillo2024bayesian}.\looseness=-1

Another common approach is to analyze the distribution of prediction errors and use it to create an interval by setting a threshold for the distance between the target and the prediction at a suitable quantile level. However, it is challenging to obtain a closed-form distribution for most prediction algorithms. Standard guarantees are achieved by limiting the analysis to situations where prediction errors can be modelled using a normal distribution e.g when applying the maximum likelihood principle \citep{van2000asymptotic}. Similarly, non-parametric approach \citep{tsybakov2009nonparametric} or Bootstrap confidence sets are only valid in the asymptotic regime or under strong regularity assumptions on the distribution; see \citep[section 3.5]{wasserman2006all} to estimate convergence rates. Overall, the current guarantees require strong regularity assumptions on the underlying distribution of the data or are only valid asymptotically.
Our objective is to depart from these classical regularity assumptions and construct confidence regions in the parameter space that have a probabilistic guarantee of containing the true parameters. In the real world, data is limited and the noise can take unknown, heteroskedastic functional forms. Our confidence regions should possess:

\qquad \emph{finite-sample validity under minimal assumptions on the data distribution.} \looseness=-1

Conformal Prediction (CP) methods \citep{Shafer_Vovk08}, provide a set of plausible values for the output $Y$ with a given input $X$. 
% The process involves a conformity function that essentially indicates the similarity between a candidate output and the corresponding input. The confidence set is obtained by selecting all outputs with sufficient conformity.
Unlike preceding methodologies, its uncertainty set is valid with finite sample sizes and under the sole assumption that one has access to similarly sampled data whose joint distribution is permutation invariant \textit{e.g.}, iid.
Nevertheless, CP offers limited to no insights into the underlying data generation process. Specifically, if $(X, Y)$ is sampled from a parameterized distribution $\mathbb{P}_{\theta_\star}$,  how can we obtain a confidence set for the true parameter $\theta_\star$ under similarly mild assumptions? \looseness=-1

\paragraph{Background}
We consider an explicit relationship between the input and output data
\begin{equation}
\label{eq:model}
Y= f_{\theta_\star}(X) + \xi,
\end{equation}

where the input $X$ takes values in $\mathbb{R}^d$, while the target (or output) $Y$ and the noise $\xi$ take values in $\mathbb{R}$. We denote $\theta_\star$ the ground truth parameter of the model.
Conformal prediction builds uncertainty sets for the noisy observation $Y$ for any tolerance level, i.e. CP builds a set-valued mapping $\Gamma$ such that
\begin{equation}
\label{eq:conformal_cov_base}
\mathbb{P}\big(Y \in \Gamma(X)\big) \geq 1 - \alpha, \quad \alpha \in (0, 1)
\end{equation}
where the probability is taken over $(X, Y)$, and the observed batch $\mathcal{D}_o = \{X_{i}^{obs}, Y_{i}^{obs}\}$ which is used to construct $\Gamma$; it is typically a training sample of the data.

We additionally consider a separate \emph{unlabelled} sample $\mathcal{D}$ of size $n$ of the data
$\mathcal{D}=\{X_1,\ldots,X_n\}.$
Our goal is now to construct a confidence region $\Theta \subset \mathbb{R}^d$ over the parameter space such that
\begin{equation}
\label{eq:target-cov}
    \mathbb{P}_{\mathcal{D}_0,\mathcal{D}}(\theta_\star\in\Theta)\geq 1-\beta, \quad \beta\in(0,1)
\end{equation}
In practice, the distribution of the noise $\xi$ is \textit{unknown}; and assumptions about it is a modeling question, which inherently varies based on the analyst's methodology and prior knowledge about the underlying phenomena being studied.
Popular settings where the noise can be supposed gaussian with identifiable (unique) ground-truth $\theta_\star$ is well studied. Standard solutions based on central limit theorem offers accurate confidence sets valid both in asymptotic and finite sample regime. Recent work \citep{wasserman2020universal} propose a significant extension by assuming continuous noise and availability of a likelihood function i.e. knowing explicitly the shape of the distribution.
Significant research efforts have been dedicated to modeling phenomena with heavy-tailed characteristics, wherein the noise lacks finite moments, rendering classical methods reliant on the central limit theorem inapplicable. Examples include the \textit{Voigt profile}, which results from convolving a Cauchy distribution with a Gaussian distribution, with applications in physics \citep{balzar1993x, chen2015optimum, pagnini2010evolution, farsad2015stable}. Similarly, the \textit{Slash distribution}, \citep{rogers1972understanding, alcantara2017slash}, is instrumental for simulating heavy-tailed phenomena. In summary, these methods are designed to model uncertainty yet are fundamentally dependent on properties of the noise that are themselves uncertain. This apparent contradiction calls into question their validity in real-world settings, and motivates our proposed methodology.

%In short, failing to capture the right regularity leads to non-valid uncertainty quantification methods. 

% We employ a strategy that doesn't rely heavily on prior knowledge on the shape of the distribution of the noise and remains valid e.g. accross symmetric noise scenarios simultaneously.

\paragraph{Contributions}

We summarize our contributions as follows:

\begin{enumerate}
\item We extend the validity of conformal prediction methodologies to build uncertainty sets for the noise-free output $f_{\theta_\star}(X)$.
\item We introduce \texttt{Conformal Confidence Regions} (\texttt{CCR}), a method to aggregate the CP intervals of multiple \emph{unlabelled} inputs $X_1, \ldots, X_n$ to construct a confidence region for the ground-truth parameter $\theta_\star$.
\item With minimal assumptions on the noise, we provide finite-sample valid coverage guarantees for \texttt{CCR} when the CP intervals are constructed by a black-box methodology, and improved guarantees when the CP intervals originate from the usual split CP method.
%\item In the particular case where the model is linear, we show that the derived confidence regions %can be viewed as the feasible set of a mixed integer linear program (MILP), allowing linear %objectives to be optimized. This permits downstream applications for 
%obtaining confidence intervals on each coordinate of the model parameter. 
\item We propose an approach for regression with abstention leveraging our uncertainty estimation, and we provide finite sample valid probability bounds.
\item We validate \texttt{CCR} empirically by constructing confidence regions for the parameter of a linear distribution $f_{\theta_\star}(X)=\theta_\star^\top X$ under challenging noise distributions, and compare to recent methodologies.

\end{enumerate}

Theoretical proofs of our results are all provided in \Cref{appendix:proof}.

\section{Confidence Set for Noise-free Outputs}

The building blocks of \texttt{CCR} are prediction intervals with finite-sample coverage over the noise-free outputs $f_{\theta_\star}(X)$. In contrast, CP only provides coverage guarantees over the noisy output $Y$. We show that the prediction intervals of CP also have coverage guarantees for the noise-free outputs.

\begin{assumption}\label{assum:gamma_is_interval}
The prediction set 
$\Gamma(X) = [A(X), B(X)]$ is an interval satisfying \Cref{eq:conformal_cov_base} for some $\alpha$.
For conformal prediction sets that aren't intervals, we consider their convex hull instead.
\end{assumption}
We define the noise dependent quantity that describes class of distribution we can handle:
\begin{align*}
d_\xi(x) &= \min\bigg(\mathbb{P}\left(\xi\geq 0\mid X=x\right), \,\mathbb{P}\left(\xi\leq 0 \mid X=x\right) \bigg) \,.
\end{align*}
\begin{assumption}\label{assum:noise}
We suppose $\displaystyle b \coloneqq \inf_{x\in \text{dom}(X)} d_\xi(x) >0$.
\end{assumption}
Symmetric noise (e.g. Gaussian) as in \citep{Csaji_2015}, implies $b=0.5$. It is sufficient for $\xi$ to have a median of $0$ everywhere to ensure $b=0.5$. Our approach, which provides coverage guarantees based on $b$, offers greater flexibility regarding data distribution compared to previous methods. Intuitively, $b$ reflects the tolerance for model misspecification, with guarantees deteriorating as $b$ decreases. An issue with classical approaches arises when the noise distribution's mean does not exist. Distributions lacking a mean still possess quantiles and our primary assumption is a bound on the quantile (e.g., median $0$), without assuming zero mean noise. Examples as the Cauchy distribution, with infinite variance, makes the central limit theorem inapplicable, yet our methods remain valid.
\begin{restatable}{proposition}{noisefreecov}
\label{prop:noise-free-cov}
Under the model in \Cref{eq:model}, \Cref{assum:gamma_is_interval} and \Cref{assum:noise}, it holds
\begin{equation}
\label{eq:noise-free-cov}
    \mathbb{P}_{\mathcal{D}_0,X}\left(f_{\theta_\star}(X) \in \Gamma(X) \right)\geq 1-\frac{\alpha}{b} \, .
\end{equation}
\end{restatable}

\Cref{prop:noise-free-cov} extends conformal guarantees to noise-free outputs, albeit with reduced coverage. Surprisingly, noise-free coverage can be lower than with noise due to adversarial scenarios where $\Gamma(X)$ often narrowly misses $f_{\theta_\star}(X)$ (we provide an explicit example in \Cref{appendix:counterexample}). Empirical evidence (\Cref{table:cov-noisy}, \Cref{table:cov-noisy-ext}) consistently shows $\Gamma(X)$ covers $f_{\theta_\star}(X)$ more frequently. Across several CP methods, training sizes, noise distributions, and thousands of trials, there is not a single trial for which the noise-free output $f_{\theta_\star}(X)$ obtains lower coverage than $Y$. Similar observations from \cite{feldman2023conformal} is that CP with noisy labels attain conservative risk over the clean ground truth labels. This motivates \Cref{assumption:strong}, a stronger but realistic alternative to \Cref{prop:noise-free-cov}.
% For completeness, we provide in the appendix, an explicit but adversarial case in which this assumption fails.\looseness=-1

% \Cref{prop:noise-free-cov} gives us the desired guarantee over the noise-free output at the cost of some coverage. It may be surprising that the resulting coverage would be lower for the noise-free output than with noise. This is due to adversarial scenarios where $\Gamma(X)$ nearly misses $f_{\theta_\star}(X)$ with large probability. In practice, the coverage of $\Gamma(X)$ will nearly always be larger for $f_{\theta_\star}(X)$, as clearly illustrated in \Cref{table:cov-noisy} and \Cref{table:cov-noisy-ext}:  
% This observation motivates us to optionally consider an assumption stronger than \Cref{prop:noise-free-cov}, which is not theoretically guaranteed but reasonable in practice.

\begin{assumption} 

\label{assumption:strong}
The noise-free outputs have the same coverage guarantee as the noisy outputs,
\begin{equation}
    \mathbb{P}_{\mathcal{D}_0,X}(f_{\theta_\star}(X)\in\Gamma(X))\geq 1-\alpha \,.
\end{equation}
\end{assumption}
To simplify, we denote $\alpha'=\frac{\alpha}{b}$ in the absence of \Cref{assumption:strong} and $\alpha'=\alpha$ under \Cref{assumption:strong}.

\begin{table}[ht]
\begin{tabular}{ll}
\hspace{-1cm}
\begin{minipage}[t]{0.5\textwidth}
        \centering
        \vspace{-6cm}
        \captionof{table}{Noise Types}
        \label{tab:noise-types}
        \resizebox{\textwidth}{!}{
        \begin{tabular}{@{} lccc @{}}
        \toprule
        Name & Notation & Distribution \\
        \midrule
        additive Gaussian & $\xi_{aG\,}$ & $\mathcal{N}(0,1)$ \\
        multiplicative Gaussian & $\xi_{mG}$ & $\mathcal{N}(0,\theta_\star^\top X)$  \\
        outliers & $\xi_{O\phantom{G}}$ & $\mathcal{P}_O$  \\
        discrete & $\xi_{D\phantom{G}}$ & $\mathcal{P}_D$  \\
        \bottomrule
        \end{tabular}
        }
        \caption*{With $\mathcal{P}_O = 0.9 \times \mathcal{N}(0,10) + 0.1 \times \mathcal{N}(0,0.05)$ and $\mathcal{P}_D(\xi=-0.5) = \mathcal{P}_D(\xi=0.5) = 0.5.$}
        
        \vspace{0.5cm} % Adjust the space between the two tables if necessary

        \captionof{table}{Coverage comparison of split CP for linear model, averaged over $1000$ trials of $1000$ test samples each in dimension $d=50$. The last column indicates how many trials had better coverage for $Y$ than for $f_{\theta_\star}(X)$. An extended version is presented in \Cref{table:cov-noisy-ext}.}
        \label{table:cov-noisy}
        \resizebox{\textwidth}{!}{
        \begin{tabular}{@{} lcccc @{}}
        \toprule
        \multirow{2}{*}{Train size} & \multirow{2}{*}{Noise} & \multirow{2}{*}{$f_{\theta_\star}(X)$} & \multirow{2}{*}{$Y$} & \multirow{2}{*}{n\_trials} \\
        & & & & \\
        \midrule
        \multirow{2}{*}{40} & Gaussian & $91.1\%$ & $91.0\%$ & 0/1000 \\
        & Outliers & $93.4\%$ & $91.0\%$ & 0/1000 \\
        % \midrule
        \multirow{2}{*}{100} & Gaussian & $96.8\%$ & $90.9\%$ & 0/1000 \\
        & Outliers & $94.0\%$ & $91.4\%$ & 0/1000\\
        \bottomrule
        \end{tabular}
        }
    \end{minipage} & 
    \begin{minipage}[t]{0.6\textwidth}
        \centering
        \includegraphics[width=\textwidth]{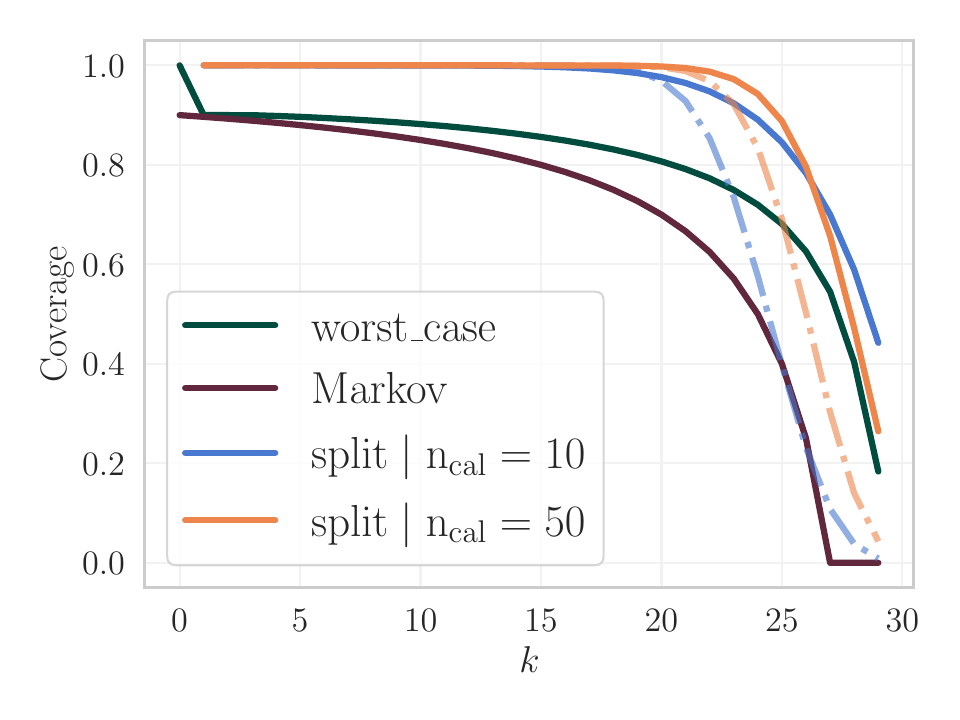}
        \vspace{0.1cm}
        \captionof{figure}{Guaranteed coverage of $\Theta_k$ as a function of $k$ from the lower bounds in \Cref{sec:confidence_sets}, under \Cref{assumption:strong}, with $n=30$. The dashed line corresponds to the PAC bounds with $\delta=0.1$. In the case of Markov, we can sample $k$ uniformly in $[k_{\rm{Markov}},n]$ without loss of guarantee so we are only plotting the coverage against the worst-case $k$.}
        \label{fig:k-cov}
    \end{minipage}
\end{tabular}
\end{table}

\section{Confidence Set for the Model Parameter}
\label{sec:confidence_sets}

Having established a confidence set for noise-free outputs, extending it to model parameters involves selecting all the parameters that result in the same confidence interval. Intuitively, this allows us to propagate uncertainty set from the outputs to the underlying model parameters. We define
$$
\Theta(X) = \{\theta \in \mathbb{R}^d : f_{\theta}(X) \in \Gamma(X)\}.
$$
% which satisfies 
% $$
% \mathbb{P}(\theta_\star \in \Theta(X)) \geq 1 - \frac{\alpha}{b}
% $$
% where $\mathbb{P} \equiv \mathbb{P}_{\mathcal{D}_0 \cup (X, Y)}$ the joint distribution on $\mathcal{D}_0$ and any $(X, Y) \in \mathcal{D}$.

While having finite sample validity under weak assumption on the noise, the set $\Theta(X)$ might be large. For linear models, $f_{\theta}(X) = \theta^\top X $, we obtain
$\Theta(X) = \{\theta \in \mathbb{R}^d : A(X) \leq \theta^\top X \leq B(X)\},$ 
which is unbounded for dimension larger than $1$.
A potential strategy would be to leverage each input $X_i \in \mathcal{D}$ that yields a valid confidence set $\Theta(X_{i})$. The most straightforward method for combining confidence sets from different test points is to intersect them. As such, we define
\[
\Theta_n = \bigcap_{i=1}^{n} \Theta(X_{i}).
\]
Without additional structural assumptions, one can employ the Bonferroni inequality to control the probability of the intersection:
$
\mathbb{P}_{\mathcal{D}_0,\mathcal{D}}(\theta_\star \in \Theta_n) \geq 1 - \sum_{i=1}^{n} \mathbb{P}(\theta_\star \notin \Theta(X_i))$. This results in a $(1 - n \alpha')$-confidence set and thus the coverage decreases rapidly with the number of points in $\mathcal{D}$.
% \todo[inline]{E: rewritte this with the low}
% \textit{Split strongly restricts the method used for constructing the conformal set (e.g. Full conformal set cannot be applied). This restriction is not needed for Markov and Worst-case, at the price of worse coverage (see Figure 1), thus lower $k$ which implies larger confidence regions. In some cases, the loss in coverage guarantee might be compensated by the use of advanced CP methods adequate to the setting. Also, Worst-case and Split require independence of the data in $D$, so the method to pick depend of whether these assumptions are valid.}

\paragraph{Aggregation by Voting}

To aggregate the $\Theta(X_i)$ more efficiently, we draw inspiration from \citep{cherubin2019majority, gasparin2024merging} and \citep{guille2024finite}, and define
\begin{equation}
    \Theta_k:=\left\{\theta\in\mathbb{R}^d: \sum_{i=1}^{n} \phi_i(\theta) \geq k\right\},
\end{equation}
for some $k \in \{1,\ldots n\}$ to be determined later and
 $\phi_i(\theta):=\mathbbm{1}_{f_\theta(X_i) \in\Gamma(X_i)}.$
This set collects parameters that are in some fraction of individual confidence set. 
One notes that when $k$ is equal to $n$, we require $\theta_\star$ to be covered by all sets $\Theta(X_i)$'s as in the previous intersection.
To guarantee the coverage of such region, we need to lower bound the probability that at least $k$ sets contain $\theta_\star$
\begin{equation}
    \label{eq:target}
\mathbb{P}_{\mathcal{D}_0,\mathcal{D}}(\theta_\star\in\Theta_k)  = 
\mathbb{P}_{\mathcal{D}_0,\mathcal{D}}\bigg(\sum_{i=1}^{n}\phi_i(\theta_\star)\geq k\bigg).
\end{equation}
While each $\phi_i(\theta_\star)$ has an expectation lower bounded by \Cref{eq:noise-free-cov}, it is crucial to note that the $\phi_i$ are \emph{not} independent preventing the use of the binomial law. To obtain lower bounds and then correctly select $k$, we propose different approaches, and illustrate their resulting guarantees in \Cref{fig:k-cov}.

\subsection{Fully Black-Box}

In this section we make no assumptions on the CP method and thus on the structure of $\Gamma$. We must thus consider any possible (e.g., worst-case) dependencies between the unknown indicators $\phi_i(\theta_\star)$.

\subsubsection{Randomized Markov's Inequality}

We denote $\phi_{i}^{c} = 1 - \phi_i$ and apply Markov's inequality
\begin{align*}
\mathbb{P}(\theta_\star  \notin \Theta_k) &= \mathbb{P}\left(\sum_{i=1}^{n} \phi_{i}^{c}(\theta_\star) \geq n-k+1 \right) 
\leq \frac{1}{n-k+1} \mathbb{E}\left(\sum_{i=1}^{n} \phi_{i}^{c}(\theta_\star)\right).
% 
% &= \frac{n}{n-k+1} \mathbb{P}(f_{\theta_\star}(X_i) \notin\Gamma(X_i))
\end{align*}

Furthermore, \citet{gasparin2024merging} improves the above bound by using the additively randomized Markov's inequality from \citet{ramdas2023randomized}, which stipulates that given $U\sim \mathrm{Unif}(0,a)$, and $Z$ a non-negative random variable, $\mathbb{P}_{Z,U}(Z\geq a-U)\leq {\mathbb{E}[Z]}/{a}.$
Applying it to our setting with $Z=\sum_{i=1}^{n} \phi_{i}^{c}(\theta_\star)$ and $a=n-k+1$, we obtain the following coverage bound.

\begin{proposition}
In the above settings, $\forall k\in[n]$, it holds
$$
\mathbb{P}(\theta_\star  \in \Theta_{k+\lfloor U\rfloor}) \geq 1 - \frac{n}{n-k+1}\alpha',
$$
where $\mathbb{P} \equiv \mathbb{P}_{\mathcal{D}_0,\mathcal{D},U}$ is the joint distribution of $\mathcal{D}_0$, $\mathcal{D}$ and $U\sim \mathrm{Unif}(0,n-k+1)$.
\end{proposition}
This allows us to tighten our confidence region by adding a random number of constraints $\lfloor U \rfloor$, without any loss in our coverage guarantee. To guarantee a coverage of $1-\beta$ for $\Theta_k$, one chooses
$$k_{\rm{Markov}}=\left\lfloor n+1-\frac{n}{\beta}\alpha'\right\rfloor + \lfloor U \rfloor.$$
Note that the coverage is only guaranteed marginally (ie on average) on $U$.

\subsubsection{Worst-Case Dependency}
\label{sec:worst-case}

Fortunately, the worst-case dependency between the $\phi_i(\theta_\star)$ may not be too detrimental to the coverage of $\Theta_k$, as long as they are independent conditionally on $\mathcal{D}_0$. Within this section only, we make the assumption that the inputs in $\mathcal{D}$ are independent.
% \begin{assumption}
% \label{assum:indep}
%     The inputs in $\mathcal{D}$ are independent.
% \end{assumption}
\Cref{prop:indep_is_worst_case} shows that when $k=n$, the worst-case dependency of the $\phi_i(\theta_\star)$ (in terms of coverage guarantee) is actually their independence. 

% \begin{assumption}
% We assume $X_i$ in $\mathcal{D}$ are sampled independently
% \end{assumption}

\begin{restatable}{proposition}{indepisworstcase}
\label{prop:indep_is_worst_case}
    % Under \Cref{assum:indep},
 Suppose the inputs in $\mathcal{D}$ are independent, it holds 
    $$\mathbb{P}_{\mathcal{D}_0,\mathcal{D}}\bigg(\theta_\star \in \Theta_n\bigg)\geq \left(1-\alpha'\right)^{n} .$$
\end{restatable}

While this bound will generally be too small to be used in practice, it motivates us to find the worst-case distribution of $\phi_i(\theta_\star)$ and use the resulting coverage as a lower bound.
Let us denote the coverage for $f_{\theta_\star}(X)$ conditional on $\mathcal{D}_0$ by $Q_{\star}(\mathcal{D}_0)$ and the probability of a $\mathrm{Binomial}(n, p)$ being larger than $k$ by $F_k(p)$. More precisely, we define
\begin{align*}
   &Q_{\star}(\mathcal{D}_0):=\mathbb{P}_{X}(f_{\theta_\star}(X) \in \Gamma(X) \mid \mathcal{D}_0), &F_k(p) = \sum_{i=k}^{n} \binom{n}{i} p^i (1-p)^{n-i}. 
\end{align*}
Since the $X_i$ in $\mathcal{D}$ are assumed to be independent, the events $f_{\theta_\star}(X_i) \in \Gamma(X_i)$ are conditionally independent given $\mathcal{D}_0$. We can thus use the binomial law on their sum, to obtain for any $\mathcal{D}_0$:
% \begin{align}
%     \mathbb{P}_{\mathcal{D}}(\theta_\star \in \Theta_k \mid \mathcal{D}_0) &= \mathbb{P}_{\mathcal{D}}\bigg(\sum_{i=1}^{n} \phi_i(\theta_\star) \geq k \mid \mathcal{D}_0\bigg) \, ,
%     = F_k(Q_{\star}(\mathcal{D}_0)) \nonumber \\
%     \mathbb{P}_{\mathcal{D}_0, \mathcal{D}}[\theta_\star \in \Theta_k] &= \mathbb{E}_{\mathcal{D}_0}\bigg[F_k(Q_{\star}(\mathcal{D}_0))\bigg].\label{eq:iterated_proba}
% \end{align}
% 
\begin{align}
    \mathbb{P}_{\mathcal{D}_0, \mathcal{D}}[\theta_\star \in \Theta_k] = 
    \mathbb{E}_{\mathcal{D}_0}\bigg[ \mathbb{P}_{\mathcal{D}}\bigg(\sum_{i=1}^{n} \phi_i(\theta_\star) \geq k \mid \mathcal{D}_0\bigg) \bigg] =
    \mathbb{E}_{\mathcal{D}_0}\bigg[F_k(Q_{\star}(\mathcal{D}_0))\bigg].\label{eq:iterated_proba}
\end{align}
In order to get the coverage of $\Theta_k$, we would like to compute the expectation of the above term over $\mathcal{D}_0$. The difficulty is that in the black-box setting, the distribution of $Q_\star$ over $\mathcal{D}_0$ is unknown, making such expectation intractable. As a workaround, we propose to compute the infimum of this expectation over all valid distributions of $Q_\star$.
That is, we seek to solve
% 
% \begin{equation}
% \label{eq:worst-case}
% \begin{aligned}
% \inf_{Q\in [0,1]^{\rm{dom}(\mathcal{D}_0)}} & \mathbb{E}_{\mathcal{D}_0}\left[\sum_{i=k}^{n} \binom{n}{i}Q(\mathcal{D}_0)^i(1-Q(\mathcal{D}_0))^{n-i}\right]\\
% \textrm{s.t.} \quad & \int_{\mathcal{D}_0}Q(\mathcal{D}_0)p(\mathcal{D}_0)\geq 1-\alpha'  \\
% \end{aligned}
% \end{equation}
% 
\begin{equation}
\label{eq:worst-case}
\begin{aligned}
\inf_{Q\in [0,1]^{\rm{dom}(\mathcal{D}_0)}} & \mathbb{E}_{\mathcal{D}_0}\bigg[F_k(Q(\mathcal{D}_0))\bigg]
\quad\textrm{ s.t. }\quad \mathbb{E}_{\mathcal{D}_0}\bigg[Q(\mathcal{D}_0)\bigg]\geq 1-\alpha' \, ,
\end{aligned}
\end{equation}
where the constraint is from \Cref{prop:noise-free-cov}. This infimum is taken over the infinite-dimensional variable $Q$ and is a-priori hard to solve. We make it tractable by reduction to a two-dimensional problem. \looseness=-1

\begin{restatable}{lemma}{lemmaworstcase}
\label{lemma-worst-case}
    There exists $A\subseteq \rm{dom}(\mathcal{D}_0)$, 
    $v \in [0,1-\alpha') \text{ and } u\in(1-\alpha',1],$
    s.t. the infimum of \Cref{eq:worst-case} is reached for 
$Q_{v,u,A}(\mathcal{D}_0) := v+(u-v)\mathbbm{1}_{\mathcal{D}_0\in A}$ and 
        $\mathbb{P}_{\mathcal{D}_0}(\mathcal{D}_0\in A) =\frac{1-\alpha'-v}{u-v}.$
\end{restatable}
Let us denote $m = \mathbb{P}_{\mathcal{D}_0}(\mathcal{D}_0\in A)$ obtained from \Cref{lemma-worst-case} and we define 
$$S(u,v,k) :=(1-m)F_k(v) + m F_k(u) .$$

\begin{proposition}
Using \Cref{eq:iterated_proba} and \Cref{lemma-worst-case}, we have the following lower bound of the coverage without assumptions on the distribution of $Q_\star$:

\begin{align*}
    \mathbb{P}_{\mathcal{D}_0,\mathcal{D}}(\theta_\star\in\Theta_k) 
    % \mathbb{E}_{\mathcal{D}_0}[\mathbb{P}_{\mathcal{D}}(\theta_\star\in\Theta_k|\mathcal{D}_0)] \\
    % &=\mathbb{E}_{\mathcal{D}_0}\left[\sum_{i=k}^{n} \binom{n}{i}Q(\mathcal{D}_0)^i(1-Q(\mathcal{D}_0))^{n-i}\right]\\
    &\geq \min_{u\in (1-\alpha', 1),v\in[0, 1-\alpha')} \quad S(u,v,k).
\end{align*}
% 
% \begin{align*}
%     S(u,v,k)&:=\mathbb{P}(\mathcal{D}_0\in A)\sum_{i=k}^n\binom{n}{i}v^i(1-v)^{n-i}\\
%     &+\mathbb{P}(\mathcal{D}_0\not\in A)\sum_{i=k}^n\binom{n}{i}u^i(1-u)^{n-i}\\
%     &=\sum_{i=k}^n\binom{n}{i}\bigg[(1-m)v^i(1-v)^{n-i}+mu^i(1-u)^{n-i}\bigg],
% \end{align*}
% 
\end{proposition}
We can now estimate the minimum of $S(u,v,k)$ with a fine-grained grid search over the real valued parameters $(u, v)$. With no assumptions on the construction of $\Gamma$ and its dependency on $\mathcal{D}_0$, we have that $\Theta_{k_{\rm{worst\_case}}}$ satisfies \Cref{eq:target-cov} with
$$
    k_{\rm{worst\_case}}=\max\left\{k \mid \min_{u,v} S(u,v,k) \geq 1-\beta\right\}.
$$

% 
% \begin{proof}
%     Let the following optimization problem
% \begin{equation}
% \begin{aligned}
% \min_{v,h\in [0,1]^3} \quad & h + (1-h)\left(\sum_{i=\kappa}^{K} \binom{K}{i}v^i(1-v)^{K-i}\right)\\
% \textrm{s.t.} \quad & h+(1-h)v\geq 1-\frac{\alpha}{u}  \\
% \end{aligned}
% \end{equation}
% \end{proof}
% 
\subsection{Split Conformal Prediction}

Split conformal prediction was introduced in \citep{papadopoulos2002inductive} to reduce the computational overhead when computing the uncertainty set at every new test point. In this setting, $\mathcal{D}_0$ is split between a training set $\mathcal{D}_{\rm tr}$ to learn a prediction model $\hat f$ and a calibration set $\mathcal{D}_{\rm cal}$ of size $n_{\rm cal}$. A conformity score $C$
% $C: \mathbb{R}^{2}\rightarrow \mathbb{R}$
will measure how plausible is a given input/output pair.
Split conformal prediction sets are defined as 
$\Gamma(X) = \left\{y: C(y, \hat f(X)) \geq Q_{\alpha}(\mathcal{D}_{\rm cal})\right\},$
where $Q_{\alpha}(\mathcal{D}_{\rm cal})$ is the $\alpha$-quantile of the scores evaluated on the calibration data.
% i.e. $C(Y_i, \hat f(X_i))$ for $(X_i, Y_i) \in \mathcal{D}_{\rm cal}$.
% $\Gamma$ then yields the desired guarantee of \Cref{eq:conformal_cov_base} when calibration and test points are assumed to be permutation invariant.
This structure can be leveraged to derive conformal guarantees on the confidence set. More precisely, we have
\begin{align*}
\theta_\star \in \Theta(X) &\Longleftrightarrow f_{\theta_\star}(X) \in \Gamma(X) 
\Longleftrightarrow
C(f_{\theta_\star}(X), \hat f(X) )\geq Q_{\alpha}(\mathcal{D}_{\rm cal}).
\end{align*}

Knowing the construction of $\Gamma$ from $\mathcal{D}_0$, we can use it to study the dependence of the conditional coverage on $\mathcal{D}_0$. Indeed, it was shown in \citep{vovk2012conditional, hulsman2022distributionfree} that 
\begin{equation}\label{eq:split-dep}
Q(\mathcal{D}_0):=\mathbb{P}_{X,Y}(Y\in\Gamma(X)|\mathcal{D}_0)\sim \mathrm{Beta}(i_\alpha, j_\alpha) \, ,
\end{equation}
where
$i_\alpha =\lceil (1-\alpha)(n_{\rm{cal}}+1)\rceil \text{ and }
j_\alpha = n_{\rm{cal}}+1-i_\alpha$. 
Thus, we leverage this distribution of the conditional coverage of $Y$, and we compute tighter bounds by factoring the dependency of $Q_{\star}(\mathcal{D}_0)$.
% from \Cref{prop:split-dep} to compute tighter bounds by factoring the dependency of $Q_{\star}(\mathcal{D}_0)$
% % $\mathbb{P}_{X}(f_{\theta_\star}(X)\in \Gamma(X)|\mathcal{D}_0)$ 
% on $\mathcal{D}_0$.
\begin{proposition}\label{prop:split_lower_bound}
We denote $Q'=Q$ under \Cref{assumption:strong} and $Q'=1-\frac{1-Q}{b}$ otherwise. It holds 
\begin{align*}
\mathbb{P}_{\mathcal{D}_0, \mathcal{D}}[\theta_\star \in \Theta_k] 
% &= \mathbb{E}_{\mathcal{D}_0}\bigg[\mathbb{P}_{\mathcal{D}}(\theta_\star\in\Theta_k|\mathcal{D}_0)\bigg] \\
= \mathbb{E}_{\mathcal{D}_0}\bigg[F_k(Q_\star(\mathcal{D}_0))\bigg] 
% &\geq \mathbb{E}_{\mathcal{D}_0}\bigg[F_k(Q'(\mathcal{D}_0))\bigg] \\
\geq \mathbb{E}_{Q \sim \mathrm{Beta}(i_\alpha, j_\alpha)} \bigg[ F_k(Q') \bigg] =: H(k) .
\end{align*}
We have that $\Theta_{k_{\rm{split}}}$ satisfies \Cref{eq:target-cov} where
$$k_{\rm{split}}=\max\{k \mid H(k)\geq 1-\beta\}.$$
\end{proposition}
% Here we have used \Cref{eq:iterated_proba}, and $Q_\star(\mathcal{D}_0)\geq Q'(\mathcal{D}_0)$ see \Cref{lm:conditional_cov} and Prop.\ref{prop:split-dep} in the appendix.
% The above expectation can be easily and approximated via Monte-Carlo methods with high accuracy. However, for completeness, we leverage Beta functions and 

Which can be computed in closed form, following \Cref{prop:close_form_expected_bound} given in \Cref{appendix:analytical_form}.

\begin{remark}[\textbf{PAC Bounds}]
Since valid coverage does not always holds conditional on the observed data, one might settle for a high probability guarantee rather than an expected one. We derive in \Cref{appendix:pac} a PAC guarantee: $\forall \alpha, \delta \in (0, 1)$ and $\tilde \alpha(\delta) = \alpha + \sqrt{ \frac{\log(1/\delta) }{n_{\rm cal}}}$, we show that
$$
\mathbb{P}_{\mathcal{D}_0} \bigg( \mathbb{P}_{\mathcal{D}}( \theta_\star \in \Theta_{k_{\rm{PAC}}} \mid \mathcal{D}_0 ) \geq 1 - \beta \bigg) \geq 1 - \delta \text{ with } k_{\rm{PAC}} = \max\left\{k: F_k\left(1 - \frac{\tilde \alpha(\delta)}{b}\right) \geq 1 - \beta \right\}.
$$
\end{remark}

% \subsubsection{High Probability Lower Bounds}
% Valid coverage does not always holds conditional on the observed data, so we might settle for a high probability guarantee rather than an expected one. Notably, the lower bound in \Cref{prop:noise-free-cov} also holds conditionally on $\mathcal{D}_{0}$.
% Thus, $\forall \alpha, \delta \in (0, 1)$, a Probably Approximately Correct guarantees on the coverage of the noisy outputs \citep{vovk2012conditional} translate into the noiseless output as well.
% % 
% \begin{restatable}{proposition}{paccoverage}
% \label{prop:paccoverage}
% Defining $\tilde \alpha(\delta) = \alpha + \sqrt{ \frac{\log(1/\delta) }{n_{\rm cal}}}$, we have
% % $$
% % \mathbb{P}_{\mathcal{D}_0} \bigg( \mathbb{P}_{X, Y}( f_{\theta_\star}(X) \in \Gamma(X) \mid \mathcal{D}_0 ) \geq 1 - \tilde \alpha(\delta) \bigg) \geq 1 - \delta
% % $$
% $
% \mathbb{P}_{\mathcal{D}_0} \bigg( Q_\star(\mathcal{D}_0) \geq 1 - \frac{\tilde \alpha(\delta)}{b} \bigg) \geq 1 - \delta,
% $ and\looseness=-1
% % $$
% % \mathbb{P}_{\mathcal{D}_0} \bigg( \mathbb{P}_{x, y}( y \in \Gamma(x) \mid \mathcal{D}_0 ) \geq 1 - \alpha(\delta) \bigg) \geq 1 - \delta
% % $$
% % 
% $$
% \mathbb{P}_{\mathcal{D}_0} \bigg( \mathbb{P}_{\mathcal{D}}( \theta_\star \in \Theta_{k_{\rm{PAC}}} \mid \mathcal{D}_0 ) \geq 1 - \beta \bigg) \geq 1 - \delta \text{ with } k_{\rm{PAC}} = \max\left\{k: F_k\left(1 - \frac{\tilde \alpha(\delta)}{b}\right) \geq 1 - \beta \right\} \,.
% $$
% \end{restatable}
\section{Related Work}

% Our work narurally intersects with the literature on conformal prediction and the literature on confidence regions.

\paragraph{Conformal prediction}
 Since their introduction, a lot of work has been done to improve the set of conformal predictions. As simple score function, distance to conditional mean ie $C(y, \hat y) = |y - \hat y|$ where $\hat y$ is an estimate of $\mathbb{E}(y \mid x)$ was prominently used \citep{papadopoulos2002inductive, lei2018distribution}. Instead, \citep{Romano_Patterson_Candes19} suggests estimating a conditional quantile instead and a conformity score function based on the distance from a trained quantile regressor, i.e. 
$C(y, \hat y) = \max(\hat y_{\alpha/2}(x) - y, y -\hat y_{1 - \alpha/2}(x))$ where $\hat y_{\alpha}$ are the $\alpha$-th quantile regressors. Others alternative consists in choosing an estimation of the conditional density of the outputs \citep{lei2013distribution, fong2021conformal, chernozhukov2021distributional, guha2024conformal}
Some others extension beyond iid assumptions in \citep{gibbs2022conformal, lin2022conformal, zaffran2022adaptive}. Efficient implementation of CP methods are available in \citep{taquet2022mapie}.\looseness=-1

\paragraph{Confidence regions} Building confidence regions for the parameters of a model is a well studied problem. These works generally either provide strict guarantees only in the asymptotic regime, such as boostrapping \cite{wasserman2006all}, or recent effort to build confidence regions with heteroskedastic noise \cite{doi:10.1080/01621459.2020.1831924}. Other works, especially in Bayesian statistics, provide finite-sample valid guarantees at the cost of strong assumptions on the noise. More recently, \citet{doi:10.1073/pnas.1922664117} proposed a method requiring to compute the likelihood, which generally implies knowledge of the noise. \citet{10.1214/aoms/1177728718} propose a distribution-free method, but the generated regions are not bounded, limiting their practicality. At the intersection, \citet{angelopoulos2023prediction} is finite-sample valid in the presence of strong noise assumptions, or asymptotically valid without such assumptions, but does not achieve both simultaneously. 
\begin{figure*}[t]
    \centering
    \includegraphics[trim={4cm 2cm 4cm 0}, width=\linewidth]{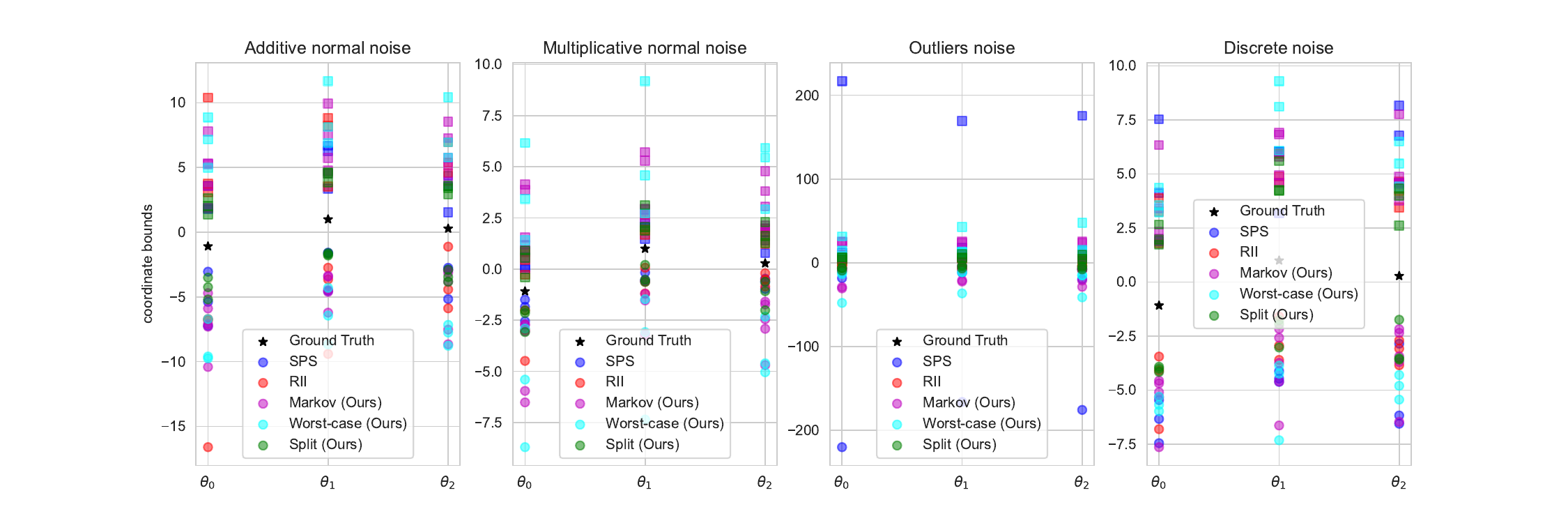}
    \vspace{0.1cm}
    \caption{Illustration of the bounds covering each coordinate of the ground-truth parameter $\theta_\star$ under various configurations of noise, for $\beta=0.1$. Squares (resp. circles) correspond to upper bounds (resp. lower bounds).}
    \label{fig:bounds}
\end{figure*}
A few contemporary works have managed to provide finite-sample valid coverage under reasonable noise assumption such as the symmetry \citep{CAMPI20051751,DALAI20071418,DENDEKKER20085024}, but they only provide methods to determine whether a given $\theta$ belongs in the confidence region. In the absence of compact representations, the applications of such sets are limited. To our knowledge, the first methods to be simultaneously finite-sample valid, bounded, reasonably constraining on the noise, and providing a compact representation are Sign-Perturbed Sums (SPS)~\citep{Csaji_2015} and RII~\citep{guille2024finite}. RII is similar to CCR in the assumptions it makes and in the way it combines confidence intervals on $f_{\theta_\star}(X)$, but instead of CP it uses so-called residual intervals
$C_i:=\left[\min\left(Y_i, \hat f\left(X_i\right)\right),\max\left(Y_i, \hat f\left(X_i\right)\right)\right]$,
with a coverage of at most $50\%$. This limitation forces RII to set $k$ to a low value, causing a relatively large confidence region and leading to expensive resolutions of the underlying MILP for downstream tasks (see \Cref{sec:MILP}). We compare CCR to RII and SPS experimentally in \Cref{sec:exps}.

\section{Applications}
\label{sec:application}

\subsection{Mixed Integer Linear Program}
\label{sec:MILP}

\begin{table}[ht]
\centering
\caption{
Coverage under noises in \Cref{tab:noise-types}, for dimension $d=3$ and $d=40$. We perform $1000$ trials in each setting, and use the least square estimator to generate split conformal predictions based on the residuals, with $\alpha=0.1$. For \texttt{CCR}, we use for each bound the largest $k$ that guarantees a coverage of at least $0.9$ under \Cref{assumption:strong}. SPS coverage is not reported when $d=40$ as its overbound region requires at least as many samples as dimensions. \\
}
\label{table:coverage}
\resizebox{0.8\columnwidth}{!}{
\begin{tabular}{ccccccccc}
\hline
\multirow{2}{*}{} & \multicolumn{4}{c}{$d=3$}                   & \multicolumn{4}{c}{$d=40$}                  \\
$\xi$          & aG & mG & O & D & aG & mG & O & D \\ \hline
SPS               & $98.9$   & $98.6$       & $99.5$     & $98.7$     & -      & -          & -        & -        \\
RII               & $90.8$   & $89.4$       & $90.5$     & $89.7$     & $88.7$   & $90.3$       & $89.8$     & $90.4$     \\
\texttt{Markov}     & $100$  & $100$      & $100$    & $100$    & $99.9$   & $100$      & $100$    & $99.9$     \\
\texttt{Worst-case} & $100$  & $100$      & $100$    & $100$    & $100$  & $100$      & $100$    & $100$    \\
\texttt{Split}      & $99.9$   & $100$      & $99.6$     & $100$    & $99.8$   & $99.7$       & $99.8$     & $99.9$     \\ \hline
\end{tabular}}
\end{table}

% \begin{table}[t]
% \centering
% \caption{Average width of the coordinate-wise confidence intervals with different noises \& the corresponding computation time.}
% \label{table:width}
% \begin{tabular}{cccccc}
% \hline
% \multirow{2}{*}{}   & \multicolumn{4}{c}{Average width}                              & \multirow{2}{*}{time ($s$)} \\
%                     & aG            & mG            & O              & D             &                           \\ \hline
% SPS                 & $7.48$          & \textbf{$2.23$} & $139$            & $10.70$         & $0.034$                     \\
% RII                 & $12.44$         & $2.77$          & \textbf{$10.63$} & $13.04$         & $12.12$                     \\
% \texttt{Markov}     & $11.37$         & $6.03$          & $29.57$          & $9.05$          & $3.58$                      \\
% \texttt{Worst-Case} & $15.56$         & $9.19$          & $44.59$          & $10.88$         & $4.56$                      \\
% \texttt{Split}      & \textbf{$6.37$} & $2.80$          & $12.28$          & \textbf{$6.53$} & $1.77$                      \\ \hline
% \end{tabular}
% \end{table}

\begin{table}[t]
\centering
\begin{minipage}[t]{0.55\linewidth}
\captionsetup{width=.9\linewidth}
\centering
\caption{Average width of the coordinate-wise confidence intervals for different noises.}
\label{table:width}
\begin{tabular}{cccccc}
\hline
                    & \multicolumn{4}{c}{Average width}                              & time ($s$) \\
                   & aG            & mG            & O              & D             &                           \\ \hline
SPS                 & $7.48$        & \textbf{$2.23$} & $139$            & $10.70$         & $0.034$                     \\
RII                 & $12.44$       & $2.77$          & \textbf{$10.63$} & $13.04$         & $12.12$                     \\
\texttt{Markov}     & $11.37$       & $6.03$          & $29.57$          & $9.05$          & $3.58$                      \\
\texttt{Worst-Case} & $15.56$       & $9.19$          & $44.59$          & $10.88$         & $4.56$                      \\
\texttt{Split}      & \textbf{$6.37$}& $2.80$          & $12.28$          & \textbf{$6.53$} & $1.77$                      \\ \hline
\end{tabular}
\end{minipage}%
\hspace{0.08\linewidth}
\begin{minipage}[t]{0.35\linewidth}
\captionsetup{width=.9\linewidth}
\centering
\caption{Rejection rate of \texttt{RII} and \texttt{CCR} on non-linear data.}
\label{table:reject}
\begin{tabular}{ll}
\hline
           & Rejection rate \\
           &                \\ \hline
           &                \\
RII        & $80\%$           \\
\texttt{Markov}     & $86\%$           \\
\texttt{Worst-Case} & $90\%$           \\
\texttt{Split}      & $100\%$          \\ \hline
\end{tabular}
\end{minipage}
\end{table}

Our method constructs the confidence region over $\theta$ by intersecting the sets $\Theta(X_i)$, and their form naturally depends on the data model $f$. We show that in the linear case $f_\theta(X)=\theta^\top X$, $\Theta_k$ can be represented as the feasible set of a MILP program, similar to \citep{guille2024finite}. This enables the optimization of linear objectives, leading to downstream applications that we discuss. 
% 
% By definition,
% $\theta\in\Theta_k \Longleftrightarrow \sum_i \phi_i(\theta)\geq k \text{ for } \phi_i(\theta) \in \{0,1\}$
% $\text{for which } \phi_i(\theta)=1 \text{ implies } A(X_i)\leq\theta^\top X_i\leq B(X_i).$
% % 
% To capture that such constraint should only be active when $\phi_i(\theta)=1$, we use the big M method \citep{wolsey2014integer,hillier2001introduction}, that sets a large constant $M$ and applying the constraint ($\phi_{i}^{c} = 1 - \phi_i$)
% $$A(X_i)- \phi_{i}^{c}(\theta) M\leq \theta^\top X_i \leq B(X_i)+\phi_{i}^{c}(\theta) M.$$
% One verifies that all the constraints for which $\phi_i(\theta)=0$ are inactive, and increase $M$ otherwise, to ensure the equivalence of the feasible set. 

\begin{restatable}{proposition}{milprepresentation}
Leveraging the classical \texttt{big M method} \citep{wolsey2014integer,hillier2001introduction}, for a sufficiently large constant $M$ we have:
\begin{align*}
    &\theta\in\Theta_k \text{ if and only if there are binary variables } (a_i)_{i\in[n]} \text{ such that }\\
    &\sum_{i=1}^n a_i \geq k, \; \text{and} \;\;
    \forall i,\, A(X_i)- M (1-a_{i}) \leq \theta^\top X_i \leq B(X_i)+ M (1-a_{i}).
\end{align*}
% \begin{align*}
%     &\sum_{i=1}^n a_i \geq k \\
%     \forall i, \quad&A(X_i)-(1-a_i)M\leq \theta^\top X_i \\
%     \forall i, \quad&\theta^\top X_i\leq B(X_i)+(1-a_i)M.
% \end{align*}
% 
 \end{restatable}
This formulation allows us to optimize linear objectives over $\Theta_k$ using MILP solvers, for instance for robust optimization in the context of Wald's minimax model \citep{wald39,wald45}. 
Alternatively, this formulation can be used to assess whether $\Theta_k$ is empty, which may occur when the predictor used by the CP methods is non-linear, and is akin to rejecting the linearity of the data. Interestingly, prior works almost always contain at least the least square estimator \citep{Csaji_2015}, even when the data is not linearly distributed. We refer to this application as hypothesis testing, because under the null hypothesis that the data is linear, the probability of $\Theta_k$ being empty (and thus not containing the ground truth parameter) is at most $\beta$.
Finally, we derive confidence intervals on specific coordinates of $\theta$ with implications for feature selection and interpretability \citep{guille2024finite}, by solving for $\max_{\theta\in\Theta_k}\theta_i$ and $\min_{\theta\in\Theta_k}\theta_i$ for a coordinate of interest $i \in \mathbb{R}^d$.

\subsection{Regression with Conformal Abstention}

Previous research efforts have proposed machine learning algorithms that refrain from making predictions when \emph{they don't know}, \citep{chow1994recognition, geifman2019selectivenet}.
\citet{zaoui2020regression} show that the optimal predictor will output $\mathbb{E}(Y \mid X=x)$ when the conditional standard deviation $\sigma^{\star}(x)$ is smaller than some prescribed threshold and abstain otherwise.
\begin{figure*}
  \centering
  \begin{subfigure}{0.49\textwidth}
  \includegraphics[width=\linewidth]{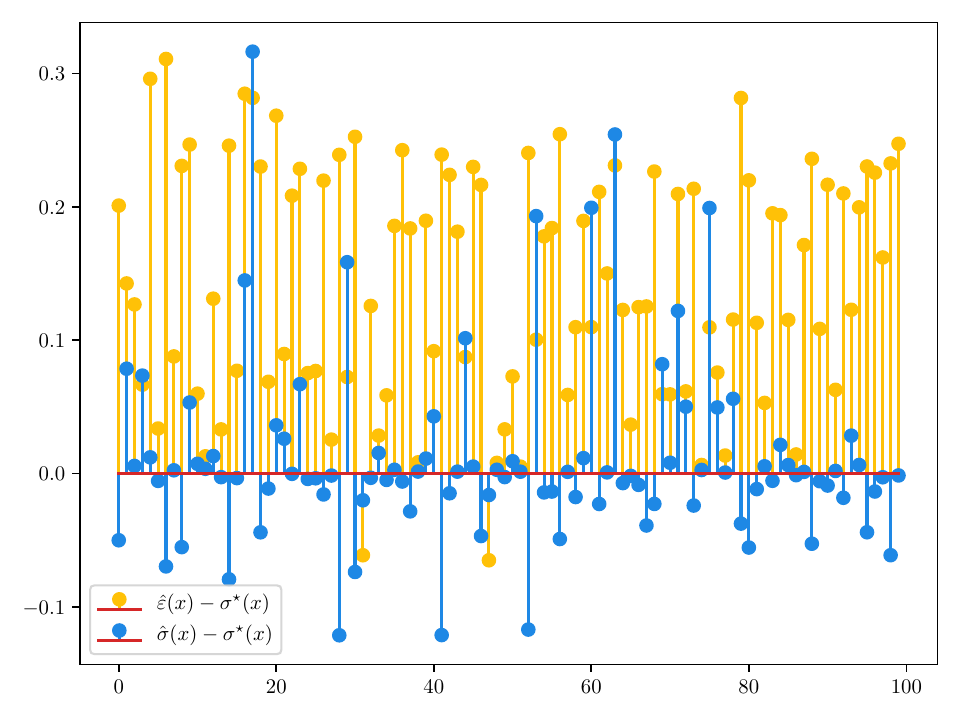}
  \caption{Estimation of standard deviation $\sigma^\star(X)$}
  \end{subfigure}
  \begin{subfigure}{0.49\textwidth}
  \label{eq:not_conservative}
  \includegraphics[width=\linewidth]{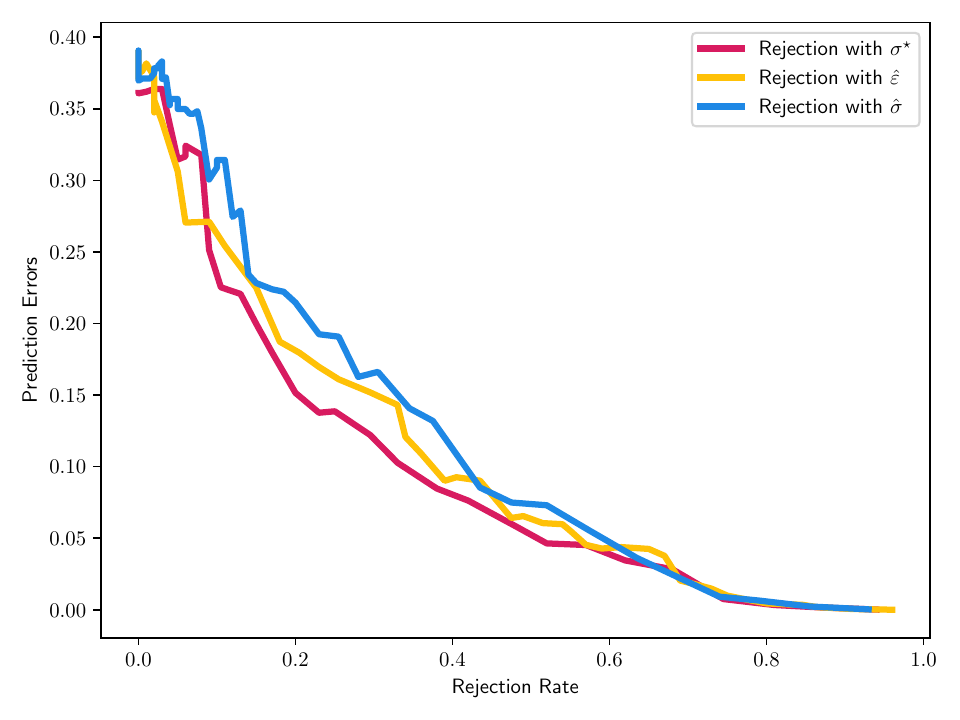}
  \caption{Trade-off between rejection frequency and MSE}
  \end{subfigure}
    \caption{Regression with rejection option with $Y_i = \sin(X_i) + \frac{\pi |X_i|}{20} \xi_i$ where $\xi_i \sim \mathcal{N}(0, 1)$.}
    \label{fig:safe_variance_estimate}
\end{figure*}
In practice, they suggest a plug-in approach that replaces the conditional mean and variance with their estimate.
Our approach leads to a natural predictor with a rejection option and a strong finite sample guarantee. One rejects the prediction when the size of the conformal set is large, reflecting high uncertainty. This was used to limit erroneous prediction in classification \citep{linusson2018classification}. We provide a non-asymptotic analysis of such approach in regression setting by showing that there is a high probability that our predictor is not empty if the amount of noise at input $X$ is sufficiently small.
The predictor with conformal abstention will predict when, $\hat \xi(X)$ the length of the conformal set $\Gamma(X)$, is smaller than a threshold.
Since with high probability, we can obtain a set that contains both $Y$ and $f_{\theta^\star}(X)$, we can expect to bound the distance between them by the length of the conformal set. \looseness=-1
\begin{restatable}{proposition}{boundabstention}
With the setting of \Cref{prop:noise-free-cov}, it holds
$$\mathbb{P}_{\mathcal{D}_0\cup (X, Y)}\bigg(|Y - f_{\theta^\star}(X)|\leq \hat \xi(X)\bigg) 
\geq 1 - \left(1 + \frac1b\right) \alpha.$$
% 
% where $\mathbb{P} \equiv \mathbb{P}_{\mathcal{D}_0\cup (X, Y)}$ is the joint distribution of $\mathcal{D}_0$ and any $(X, Y) \in \mathcal{D}$.
% \todo{Rejection rate to clarify}
% Moreover, we have the rejection rate control for any $\delta \in (0, 1)$, $\mathbb{P}(Q_\delta(\mathcal{D}_0, x) \leq \lambda_\epsilon) \geq 1 - \epsilon$ 
% for a threshold level
% $$
% \lambda_\epsilon(x) = \inf\{\lambda : \mathbb{P}\left( \mathbb{F}_{n_{\rm cal}}(\lambda, x) \geq \delta \right ) \geq 1 - \epsilon\}.
% $$
\end{restatable}

In \Cref{fig:safe_variance_estimate}, we illustrate the discrepancy between a standard deviation estimate and our safe approach. Numerical observations indicate that the estimated conditional variance frequently tends to be smaller than its target; while the length of the conformal set is often an upper bound for the standard deviation $\sigma^\star$. This means that if the optimal predictor 
% \Cref{eq:optimal_regression_rejection} 
is uncertain enough to make a prediction, we will also safely abstain; this holds for each data point.
In contrast to previous methods, our theoretical analysis holds in finite sample size, with mild assumption on the noise, and for any predictor $\hat f$. Despite being a safer approach for each instance, \Cref{eq:not_conservative} illustrates that, in average, the prediction with conformal abstention is not too conservative.

%We show in \Cref{appendix-rej} how the conformal intervals on $f_{\theta_\star}(X)$ can be leveraged to perform regression with rejection.
% 
% \begin{equation}
% \label{eq:optimal_regression_rejection}
% \mathcal{Y}^{\star}(x, \lambda) =
% \begin{cases}
% \{\mathbb{E}(Y \mid X=x)\} &\text{ if } \sigma^{\star}(x) \leq \lambda \\
% \emptyset  &\text{ otherwise }
% \end{cases}
% \end{equation}
% 
% where $\sigma^{\star}$ is the conditional standard deviation. 

\vspace{-1mm}
\section{Numerical Experiments}
\label{sec:exps}

We evaluate our method in the usual setting of linearly distributed data, and compare it to \texttt{SPS} and \texttt{RII}. We emphasize that similarly to previous works \citep{guille2024finite,Csaji_2015} we only evaluate on synthetic data, because ground truth parameters are not known for real world datasets. Therefor, the validity of coverage bounds can not be confirmed, and the provided guarantees are never valid for distributions that may not be perfectly linear, making confidence regions impossible to compare fairly. We consider four distinct types of noise $\xi$ summarized in \Cref{tab:noise-types}. In all of our experiments, unless stated otherwise, we set $\alpha=0.1$, sample $\theta_\star$ from $\mathcal{N}(0,1)$, and the $X_i$ are sampled independently from $\mathrm{Unif}([0,1]^d)$.

\paragraph{Coverage} In \Cref{table:coverage}, the display coverage of various methods. We find that even with $k$ chosen based on \Cref{assumption:strong}, \texttt{CCR} strongly overbounds $\theta_\star$, with a coverage much above the guaranteed rate. This comes from the CP intervals having a much larger coverage for $f_{\theta_\star}(X)$ than for the noisy output $Y$. While we seek tight confidence regions, it is beneficial to observe higher than expected coverage as long as the size of $\Theta$ remains small. Interestingly, we find that \texttt{SPS} is not capable of producing optimizable confidence regions when the dimension is larger than the training set size.

\paragraph{Coordinate intervals} \Cref{fig:bounds} illustrates the lower and upper bounds per coordinate inferred by optimizing over \texttt{SPS}, \texttt{RII} and \texttt{CCR}. The average width for each method and noise, as well as the corresponding computation time, are reported in \Cref{table:width}. Experiments are run in dimension $d=3$ for the sake of clarity, with a training set of $20$ points. For each method, we run $3$ independent trials (with fixed $\theta_\star$) to visualize the variance of the produced coordinate bounds. We find that while SPS is very fast and performs well on Gaussian noises, its confidence regions tend to be unreliable on less standard noises. While \texttt{RII} is versatile, its computation time for the bounds is high due to the adoption of a low $k$, resulting in more combinatorial complexity for the MILP. Furthermore, we find that Markov and worst-case bounds tend to perform worse than prior works, suggesting that the uncertainty on the distribution of the coverage for black-box CP methods is limiting. The split CP bound achieves a reasonable running time (though much slower than \texttt{SPS}), and produces confidence intervals across noises that are smaller or comparable to other methods, which supports its applicability. \looseness=-1

% \begin{table}[h]
% \centering
% \caption{Rejection rate of \texttt{RII} and \texttt{CCR} on a non-linear data relationship, using an adequate estimator.}
% \label{table:reject}
% \begin{tabular}{ll}
% \hline
%            & Rejection rate \\ \hline
% RII        & 80\%           \\
% Markov     & 86\%           \\
% Worst case & 90\%           \\
% Split      & 100\%          \\ \hline
% \end{tabular}
% \end{table}

% \begin{table}[h]
% \centering
% \caption{Rejection rate of \texttt{RII} and \texttt{CCR} on a non-linear data relationship, using an adequate estimator.}
% \label{table:reject}
% \begin{tabular}{lcccc}
% % \hline
%            & RII  & Markov & Worst case & Split \\ \cline{2-5}
% Rejection rate & 80\% & 86\%   & 90\%       & 100\% \\ \hline
% \end{tabular}
% \end{table}

\paragraph{Hypothesis testing} Following \citep{guille2024finite}, we consider the model $Y=\theta_\star^\top X + 0.5 \sin(8\pi \|X\|_2) + \mathcal{N}(0,1).$
Using the least square predictor on $Z:=(X,\sin(8\pi\|X\|_2)$, we measure whether $\Theta_k$ is  empty, rightfully rejecting the linearity hypothesis, and average this rate over $200$ trials. \texttt{CCR} outperforms all methods in rejecting linearity. Unlike \texttt{SPS}, it's noteworthy that \texttt{CCR} can reject linearity, as the least square estimator on $X$ is in the confidence region.\looseness=-1

\paragraph{Conclusion} We presented \texttt{CCR}, a method for constructing confidence regions under minimal assumptions on the noise, for which we provided finite-sample valid coverage guarantees in different settings. Additional discussions, such as on apparent limitations of our work, can be found in \Cref{appendix:discussion} \looseness=-1

\bibliographystyle{plainnat}
\bibliography{biblio}

\newpage
\appendix

\section{Additional Experiments}
\label{appendix:add-exps}
\Cref{table:cov-noisy-ext} gives an extended version of \Cref{table:cov-noisy} over a wider variety of settings, with the same conclusions.

\begin{table}[ht]
\centering
\caption{Comparison of Coverage: $\Gamma$ is build by split residual CP or Conformalized Quantile CP \citep{NEURIPS2019_5103c358} with $\alpha=0.1$ for linear $f_{\theta_\star}(X)$ and $Y$. Averaged over $1000$ trials of $1000$ test samples each in dimension $d=50$, for different training set sizes and noise distributions (see \Cref{sec:exps}). The last column indicates the number of trials where $\Gamma$ had better coverage for $Y$ than for $f_{\theta_\star}(X)$.\looseness=-1 \\}
\label{table:cov-noisy-ext}
\begin{tabular}{cccccc}
\hline
CP Method & Train Size & Noise & $f_{\theta_\star}(X)$ & $Y$ & \# Losses \\ \hline
\multirow{6}{*}{Residual} & \multirow{3}{*}{40}  & Gaussian       & $91.1\%$ & $91.0\%$ & $0/1000$ \\
                          &                      & Outliers       & $93.4\%$ & $91.0\%$ & $0/1000$ \\
                          &                      & Mult. Gaussian & $92.3\%$ & $91.4\%$ & $0/1000$ \\
                          & \multirow{3}{*}{100} & Gaussian       & $96.8\%$ & $90.9\%$ & $0/1000$ \\
                          &                      & Outliers       & $94.0\%$ & $91.4\%$ & $0/1000$ \\
                          &                      & Mult. Gaussian & $96.4\%$ & $90.7\%$ & $0/1000$ \\ \hline
\multirow{6}{*}{Quantile} & \multirow{3}{*}{40}  & Gaussian       & $92.3\%$ & $92.3\%$ & $0/1000$ \\
                          &                      & Outliers       & $95.0\%$ & $92.4\%$ & $0/1000$ \\
                          &                      & Mult. Gaussian & $92.6\%$ & $92.1\%$ & $0/1000$ \\
                          & \multirow{3}{*}{100} & Gaussian       & $90.4\%$ & $90.3\%$ & $0/1000$ \\
                          &                      & Outliers       & $93.5\%$ & $90.4\%$ & $0/1000$ \\
                          &                      & Mult. Gaussian & $90.8\%$ & $90.2\%$ & $0/1000$ \\ \hline
\end{tabular}
\end{table}

\section{Additional Discussions}

\label{appendix:discussion}

We introduced \texttt{CCR}, a method for aggregating conformal prediction intervals into finite-sample valid confidence regions on model parameters, adaptable to diverse noise types. For linear models, the resulting confidence region can be formulated as an MILP, showing competitiveness with previous approaches in various downstream applications. In the following, we discuss some of the critical drawbacks of the proposed methods.

\subsection{Limitations}

\subsubsection{Model well-specification }

Our approach begins with the explicit relation between inputs and outputs, as stated in \Cref{eq:model}. If this relation is incorrect, namely in the case of model misspecification, our claims will not hold. Nevertheless, in contrast to previous approaches, the departure from this core assumption is controlled by the parameter $b$, which is assumed to be known. Incorrectly setting this parameter can lead to unreliable results. This is to be contrasted with the assumption that the shape of the noise distribution is known, which is a significantly stronger assumption than the one made here.

\subsubsection{Inadequacy in high dimensional regime}

Our main goal is to provide systematic way of building confidence sets for model parameters for data-driven prediction methods. Our theory comes with explicit bounds that tightly quantify the coverage we can guarantee. However, good coverage does not necessarily align with efficiency (ie small enough confident region) and a trade-off between the two might be necessary. As we displayed the coverage curves in \Cref{fig:k-cov} and illustrated the boundness of our uncertainty region for linear models, we must restrict the application to small enough dimension in order to maintain a reasonable efficiency. As a consequence, our method is of limited used in modern deep learning models.

\subsubsection{Suboptimality wrt standard methods in the well behaved distribution settings}

\begin{figure}
  \centering
  \includegraphics[width=0.8\textwidth]{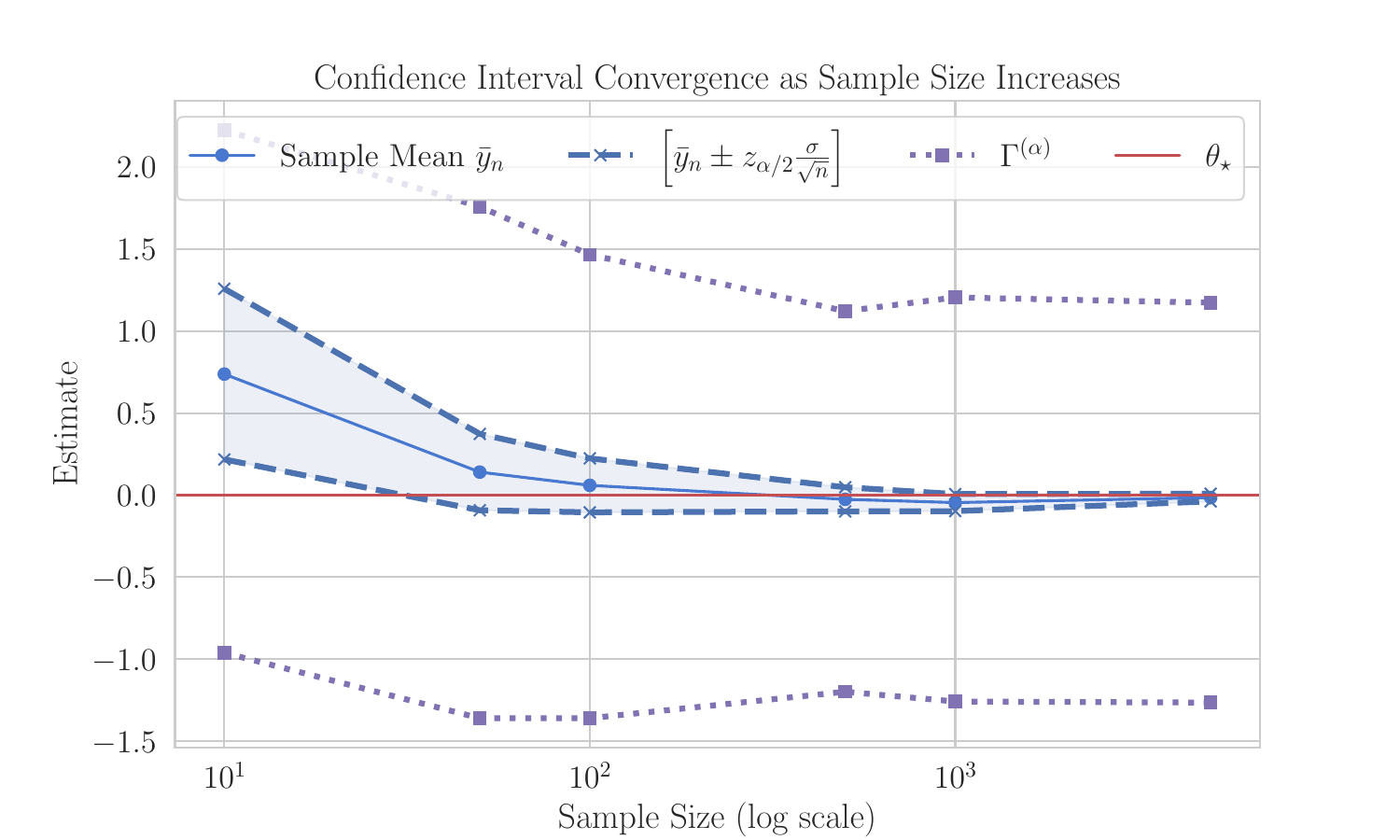}
  \caption{Limitations of \texttt{CCR} compared to classical approach when the model and the data distribution behave nicely. One can observe a great advantage of the classical method that reduces the variance and shrink to the ground-truth as we collect more data.}
  \label{fig:limitation_ccr}
\end{figure}

\begin{figure}
  \centering
  \includegraphics[width=\linewidth]{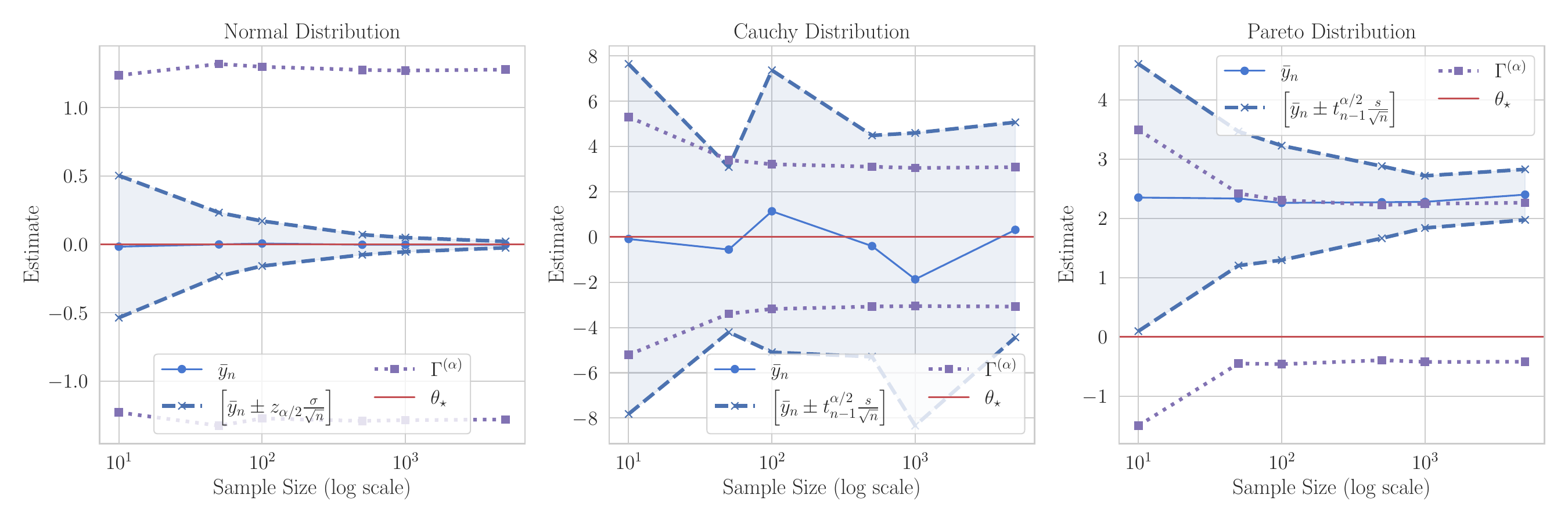}
  \caption{Different noise setting where the classical confidence set estimate the standard deviation. In all setting, we use the median as the base predictor for conformal prediction. We recall that in these examples, the classical strategy have no validity. Our approach start by establishing validity first and then proceed to improve efficiency. As one can observed in the case of Pareto distribution, the standard confidence set can be misleading and this is hard to spot with real dataset since the ground-truth is never known. By setting $b=0.5$, our method is simultaneously valid for any symmetric noise. \looseness=-1}
  \label{fig:issues_classical}
\end{figure}

Consider a sequence of iid variable $y_1, y_2 \ldots$ following a gaussian distribution $\mathcal{N}(\theta_\star, \sigma^2)$ where for simplicity we suppose that the standard deviation $\sigma > 0$ is known.
In that case, the ground-truth $\theta_\star$ can be estimated with the empirical mean $\bar{y}_n = \frac1n \sum_{i=1}^{n} y_i$ and we know exactly the distribution of the estimation error $\bar{y}_n  - \theta_\star \sim \mathcal{N}(0, \frac{\sigma^2}{n})$. Thus, a confidence interval for the model parameter is obtained as \looseness=-1

\begin{equation}\label{eq:classical_cs}
    \mathbb{P}\left(\theta_\star \in \left[\bar{y}_n \pm z_{\alpha/2} \frac{\sigma}{\sqrt{n}} \right] \right) \geq 1 - \alpha.
\end{equation}

Similar strategy can be used to sequentially build prediction sets for the outputs. Indeed, we have $\bar{y}_{n-1}  - \theta_\star \sim \mathcal{N}(0, \frac{\sigma^2}{n-1})$ and $y_n = \theta_\star + \mathcal{N}(0, \sigma^2)$. As such, we can also characterize the distribution of errors $y_n - \bar{y}_{n-1}  \sim \mathcal{N}(0, \sigma^2 \frac{n}{n - 1})$. Thus, a prediction set is obtained as

\begin{equation}\label{eq:predset_var}
    \mathbb{P}\left(y_n \in  \left[\bar{y}_{n-1} \pm z_{\alpha/2} \sigma \sqrt{\frac{n}{n-1}} \right]  \right) \geq 1 - \alpha.
\end{equation}

Following our strategy, we should exploit the knowledge $y_n = \theta_\star + \mathcal{N}(0, \sigma^2)$ and then
$$y_n - \bar{y}_{n-1}  \sim \mathcal{N}\left(0, \sigma^2 \frac{n}{n - 1}\right) \Longrightarrow \theta_\star - \bar{y}_{n-1}  \sim \mathcal{N}\left(0, \sigma^2\frac{2n - 1}{n - 1}\right)$$

and finally
\begin{equation}\label{eq:conformal_cs_var}
    \mathbb{P}\left(\theta_\star \in  \left[\bar{y}_{n-1} \pm z_{\alpha/2} \sigma \sqrt{\frac{2n - 1}{n-1}} \right]  \right) \geq 1 - \alpha.
\end{equation}

As one can see, the interval \Cref{eq:classical_cs} is much better than \Cref{eq:conformal_cs_var} as the former shrink to a point mass while the latter does not. Also, both are affected by the standard deviation $\sigma$.

However, that's not exactly what we are doing. We can avoid being affected by $\sigma$. Indeed our approach still hold even when $\sigma = + \infty$. First, notice that using \Cref{eq:predset_var} directly with \Cref{eq:noise-free-cov} reads

\begin{equation}
    \mathbb{P}\left(\theta_\star \in  \left[\bar{y}_{n-1} \pm z_{\alpha/2} \sigma \sqrt{\frac{n}{n-1}} \right]  \right) \geq 1 - 2 \alpha.
\end{equation}
This mean that explicitly exploiting the shape of the distribution instead of its sole symmetry might not be a good idea because, the set does not shrink to a point mass while still being affected, explicitly, by the standard deviation. This does not allows to handle cases where the variance is infinite.

\paragraph{Leveraging Conformal Prediction set} Now we make no further assumption beside $b=1/2$ which is true for gaussian distribution. Our base predictor is the empirical mean and the score function is simply the distance to the mean (any estimator can be used e.g. median can be used) ie 
$$E_i(z) = |y_i - \bar{y}(z)|, \quad E_n(z) = |z - \bar{y}(z)|, \text{ where } \bar{y}(z) = \frac{1}{n}\left(\sum_{i=1}^{n-1} y_i + z \right)
$$
It is easy to see that $E_n(z) = |z - \bar{y}(z)| = \frac{n-1}{n} |z - \bar{y}_{n-1}|$. Defining $Q_{1-\alpha}(z)$ as the $(1-\alpha)$-quantile of $\{E_1(z), \ldots, E_{n-1}(z), E_n(z)\}$, the $(1-\alpha)$-conformal set for $y_n$ is given by:
$$
\Gamma^{(\alpha)} = \left\{z \in \mathbb{R} : |z - \bar{y}_{n-1}| \leq \frac{n}{n-1} Q_{1-\alpha}(z) \right\}
$$
Our result \Cref{eq:noise-free-cov} (resp. The split CP version) reads
$$
\mathbb{P}\left(\theta_\star \in \Gamma^{(\alpha)} \right) \geq 1 - 2 \alpha, \qquad 
\mathbb{P} \left(\theta_\star \in [\bar{y}_{\mathrm{tr}} \pm Q_{1-\alpha}(\mathrm{cal})] \right) \geq 1 - 2 \alpha.
$$

Even if for a given dataset, the radius $Q_{1-\alpha}(\mathrm{cal})$ is always finite, it will be large if the variance of the underlying distribution is large. This occurs because the base predictor may have difficulty accurately predicting the targets, resulting in large score functions. 

% \todo{We are actually intersecting these sets ...}

% \paragraph{Intersecting oracle sets}
% The reasonable oracle version, that exploits the shape of the distribution would proceed as follow:
% $$
% y_i - \theta_\star = \mathcal{N}(0, \sigma^2) \implies \theta_\star \in \Theta_i = \left[y_i \pm z_{\alpha/2} \sigma \right]
% $$
% \todo{and we are supposed to consider $\cap \Theta_i$}

% \todo[inline]{write it out for Cauchy distribution, we know the closed form of the quantile}
\newpage
\section{Proofs}
\label{appendix:proof}

In this section, we provide detailed proofs for the propositions outlined in the main paper.

\subsection{Coverage of the noiseless output}
% \begin{proposition}
% Let $\xi$ be any noise, we assume an additive model 
% $$ y = f_{\theta_\star}(x) + \epsilon $$
% Then it holds
% $$
% \frac{\mathbb{P}(y \in [a, b]) - \max( \mathbb{P}(\xi > 0), \mathbb{P}(\xi < 0))}{1 -  \max( \mathbb{P}(\xi > 0), \mathbb{P}(\xi < 0))} \leq  \mathbb{P}( f_{\theta_\star}(x) \in [a, b])
% $$
% \end{proposition}

% \begin{proof}
% For simplicity, we denote $y_\star = f_{\theta_\star}(x)$ and $E = y \in [a, b]$. We have
% % 
% \begin{align*}
% \mathbb{P}(E) &= \mathbb{P}(E, y_\star \in [a, b]) + \mathbb{P}(E, y_\star < a) + \mathbb{P}(E, y_\star > b) \\
% % 
% &\leq \mathbb{P}(y_\star \in [a, b]) + \mathbb{P}(E \mid y_\star < a)\mathbb{P}(y_\star < a) + \mathbb{P}(E \mid y_\star > b)\mathbb{P}(y_\star > b)
% \end{align*}
% % 
% Now, one observes that $\mathbb{P}(E \mid y_\star < a) = \mathbb{P}(\xi \in [a - y_\star, b - y_\star] \mid y_\star < a) \leq \mathbb{P}(\xi > 0)$.

% Similarly, $\mathbb{P}(E \mid y_\star > b) \leq \mathbb{P}(\xi < 0)$. Hence, we conclude using
% % 
% \begin{align*}
% \mathbb{P}(E) &\leq \mathbb{P}( y_\star \in [a, b]) + \mathbb{P}(\xi > 0) \mathbb{P}(y_\star < a) + \mathbb{P}(\xi < 0) \mathbb{P}(y_\star > b) \\
% % 
% &\leq \mathbb{P}( y_\star \in [a, b]) + \max( \mathbb{P}(\xi > 0), \mathbb{P}(\xi < 0)) \mathbb{P}( y_\star \notin [a, b])
% \end{align*}
% \end{proof}

% \newpage

\noisefreecov*

\begin{proof}

For simplicity, let us denote $\Gamma(X):=[A(X),B(X)]$ and $E:=Y\in\Gamma(X)$. We have
\begin{align*}
    \mathbb{P}(E)&=\mathbb{P}\left(E|f_{\theta_\star}(X)\in \Gamma\left(X\right)\right)\mathbb{P}\left(f_{\theta_\star}(X)\in\Gamma\left(X\right)\right) \\
    &+\mathbb{P}\left(E|f_{\theta_\star}(X)<A\left(X\right)\right)\mathbb{P}\left(f_{\theta_\star}(X)<A\left(X\right)\right) \\
    &+\mathbb{P}\left(E|f_{\theta_\star}(X)>B\left(X\right)\right)\mathbb{P}\left(f_{\theta_\star}(X)>B\left(X\right)\right) 
\end{align*}

If $f_{\theta_\star}(X)<A(X)$, then $f_{\theta_\star}(X)+\xi\in\Gamma(X)\Longrightarrow \xi>0$. Thus,
\begin{equation*}
    \mathbb{P}(E|f_{\theta_\star}(X)<A(X))\leq \mathbb{P}(\xi>0|f_{\theta_\star}(X)<A(X))\leq 1-\mathbb{P}(\xi\leq 0|f_{\theta_\star}(X)<A(X)).
\end{equation*}
From \Cref{assum:noise}, $\forall x\in\mathbb{R}^d, \mathbb{P}(\xi\leq 0|X=x)\geq b$ and thus 
\begin{equation*}
\mathbb{P}(E|f_{\theta_\star}(X)<A(X))\leq 1-b.
\end{equation*}
Similarly,
\begin{equation*}
    \mathbb{P}(E|f_{\theta_\star}(X)>B(X))\leq 1-b
\end{equation*}

Injecting above, and upper bounding $\mathbb{P}\left(E|f_{\theta_\star}(X)\in \Gamma\left(X\right)\right)$ with $1$, we obtain
\begin{align*}
    \mathbb{P}(E)&\leq \mathbb{P}\left(f_{\theta_\star}(X)\in\Gamma\left(X\right)\right) + (1-b)\bigg(\mathbb{P}\left(f_{\theta_\star}(X)<A\left(X\right)\right) + \mathbb{P}\left(f_{\theta_\star}(X)>B\left(X\right)\right)\bigg) \\
    &\leq \mathbb{P}\left(f_{\theta_\star}(X)\in\Gamma\left(X\right)\right) + (1-b)(1-\mathbb{P}\left(f_{\theta_\star}(X)\in\Gamma\left(X\right)\right))\\
    &\leq b\times\mathbb{P}\left(f_{\theta_\star}(X)\in\Gamma\left(X\right)\right)+1-b.
\end{align*}

Or equivalently,

\begin{equation*}
    \mathbb{P}\left(f_{\theta_\star}(X)\in\Gamma\left(X\right)\right)\geq \frac{\mathbb{P}(E)-(1-b)}{b}\geq 1-\alpha'
\end{equation*}

Which concludes the proof.
\end{proof}

\subsection{Lower bounds on the expected coverage}

\indepisworstcase*

\begin{proof}
    It holds
\begin{equation*}
    \mathbb{P}_{\mathcal{D},{\mathcal{D}_{0}}}(\theta_\star\in\Theta_n)=\mathbb{E}_{\mathcal{D}_{0}}[\mathbb{P}_{\mathcal{D}}(\theta_\star\in\Theta_n|{\mathcal{D}_{0}})],
\end{equation*}

and 

\begin{equation}
\label{eq-proof-4.3-1}
    \mathbb{P}_{\mathcal{D}}(\theta_\star\in\Theta_n|{\mathcal{D}_{0}})=\mathbb{P}_{\mathcal{D}}(\forall i\in [n], f_{\theta_\star}(X_i) \in \Gamma(X_i)|{\mathcal{D}_{0}}).
\end{equation}

Since the $X_i$ are assumed to be sampled independently, it is easy to see that conditionally on ${\mathcal{D}_{0}}$, $f_{\theta_\star}(X_i)\in\Gamma(X_i)$ is independent from $f_{\theta_\star}(X)_j\in\Gamma(X_j)$ for $i\neq j$. Thus, \Cref{eq-proof-4.3-1} can be rewritten as

\begin{align*}
    \mathbb{P}_{\mathcal{D},{\mathcal{D}_{0}}}(\theta_\star\in\Theta_n)&=\mathbb{E}_{\mathcal{D}_{0}}\bigg[\prod_{i=1}^n \mathbb{P}_{X_i}(f_{\theta_\star}(X_i)\in\Gamma(X_i)|{\mathcal{D}_{0}})\bigg] \\
    &=\mathbb{E}_{\mathcal{D}_{0}}\bigg[\mathbb{P}_X(f_{\theta_\star}(X)\in\Gamma(X)|{\mathcal{D}_{0}})^n\bigg].
\end{align*}
Now it only remains to use Jensen's inequality to find that
\begin{align*}
    \mathbb{E}_{\mathcal{D}_{0}}\bigg[\mathbb{P}_X(f_{\theta_\star}(X)\in\Gamma(X)|{\mathcal{D}_{0}})^n\bigg]&\geq \bigg(\mathbb{E}_{\mathcal{D}_{0}}\bigg[\mathbb{P}_X(f_{\theta_\star}(X)\in\Gamma(X)|{\mathcal{D}_{0}}\bigg]\bigg)^n \\
    &=\bigg(\mathbb{P}_X(f_{\theta_\star}(X)\in\Gamma(X)\bigg)^n \\
    &\geq (1-\alpha')^n.
\end{align*}
Concluding the proof.
\end{proof}

\lemmaworstcase*

\begin{proof}

    Let $$S(Q):=\mathbb{E}_{\mathcal{D}_0}\left[\sum_{i=k}^{n} \binom{n}{i}Q(\mathcal{D}_0)^i(1-Q(\mathcal{D}_0))^{n-i}\right]$$
    
    The structure of the proof is as follow. We first prove that given a feasible $Q$, it is possible to shrink the support set of $Q$ around two values $u$ and $v$ while preserving $\mathbb{E}[Q]$ and preserving or decreasing $S(Q)$.

    Let $Q\in [0,1]^{dom(\mathcal{D}_0)}$ such that $\mathbb{E}[Q]\geq 1-\alpha'$.

    We want to prove that there exists $0\leq v\leq u\leq 1$ and $Q'\in \{v,u\}^{dom(\mathcal{D}_0)}$ such that $\mathbb{E}[Q']\geq 1-\alpha'$ and $S(Q')\leq S(Q)$.
    If $Q$ is constant, then we can take $Q'=Q$ with $u=v=\mathbb{E}[Q]$. Else, let
    $$V^+=\{\mathcal{D}_0|Q(\mathcal{D}_0)\leq \mathbb{E}[Q]\}$$ and
    $$V^-=\{\mathcal{D}_0|Q(\mathcal{D}_0)> \mathbb{E}[Q]\}.$$

    Since $Q$ is not constant, there are some $\mathcal{D}_0$ for which it is larger than its expectation, and some for which it is smaller, thus $V^+\neq \emptyset$ and $V^-\neq \emptyset$.
    Let us denote 
    $$d(Q):=d^+ +d^-,$$
    with $d^+:=\max_{\mathcal{D}_0\in V^+}(Q(\mathcal{D}_0))-\min_{\mathcal{D}_0\in V^+}(Q(\mathcal{D}_0))$ and $d^-:=\max_{\mathcal{D}_0\in V^-}(Q(\mathcal{D}_0))-\min_{\mathcal{D}_0\in V^-}(Q(\mathcal{D}_0))$.
    Let us consider $d^+\geq d^-$. We now split $V^+$ into $V_1^+:=\{\mathcal{D}_0\in V^+|Q(\mathcal{D}_0)\leq \min_{\mathcal{D}_0\in V^+}(Q(\mathcal{D}_0))+\frac{d^+}{2}\}$ and $V_2^+:=\{\mathcal{D}_0\in V^+|Q(\mathcal{D}_0)> \min_{\mathcal{D}_0\in V^+}(Q(\mathcal{D}_0))+\frac{d^+}{2}\}$. We now want to find $V_1^-\subseteq V^-$ such that \begin{equation}
        \label{eq:lemma:1}
    \mathbb{E}[Q|\mathcal{D}_0\in V_1^-\cup V_1^+]=\mathbb{E}[Q]
    \end{equation}

$V_1^-$ can for instance be constructed by taking $\mathcal{D}\in V^-$ and considering the set $V^-\cap \mathcal{B}(\mathcal{D}, r)$ where $\mathcal{B}(c, r)$ represents the ball of center $c$ and radius $r$ for the euclidean metric, viewing $\mathcal{D}_0$ as an element of $\mathbb{R}^{nd}$, and increasing $r$ until reaching $\mathbb{E}[Q]$. In that construction, $\mathbb{E}[Q|\mathcal{D}_0\in V_1^-\cup V_1^+]$ decreases monotonously and continuously (since $\mathcal{D}_0$ is a continuous R.V.), and for $r=0$ it is the expectation over $V_1^+$ which is larger than $\mathbb{E}[Q]$ by construction, while for $r$ sufficiently large we get $\mathbb{E}[Q|\mathcal{D}_0\in V^-\cup V_1^+]$ which is smaller than $\mathbb{E}[Q]$ because we are only removing elements of $V_2^+$ which are all larger than $\mathbb{E}[Q]$. Hence, by intermediate value theorem, there is indeed an $r$ such that \Cref{eq:lemma:1} is satisfied.

We now consider $V_1:=V_1^-\cup V_1^+$ and $V_2:=V_2^-\cup V_2^+$. 

Let
$$S(Q|V_{1,2}):=\mathbb{E}_{\mathcal{D}_0}\left[\sum_{i=k}^{n} \binom{n}{i}Q(\mathcal{D}_0)^i(1-Q(\mathcal{D}_0))^{n-i}\bigg|\mathcal{D}_0\in V_{1,2}\right]$$
Let us assume for instance that $S(Q|V_1)\geq S(Q|V_2)$.

We now want to define a projection $T$ from $V_1$ to $V_2$ preserving relative probabilities. For convenience, let $\leq$ be an order on $dom(\mathcal{D}_0)$, such as the dictionary order, $F_{V_1}$ the cumulative distribution function on $V_1$ for that order, and $F^{-1}_{V_2}$ the quantile function on $V_2$. We define $T(\mathcal{D}_0):=F^{-1}_{V_2}F_{V_1}(\mathcal{D}_0)$.
We then have
\begin{align*}
    \mathbb{P}(Q(T(\mathcal{D}_0))=q|\mathcal{D}_0\in V_1) &= \mathbb{P}(Q(F_{V_2}^{-1}F_{V_1}(\mathcal{D}_0))=q|\mathcal{D}_0\in V_1)\\
    &= \mathbb{P}_{s\sim Unif[0,1]}(Q(F_{V_2}^{-1}(s)=q) \\
    &= \mathbb{P}(Q(\mathcal{D}_0)=q|\mathcal{D}_0\in V_2)
\end{align*}
And thus \begin{align*} 
\mathbb{E}[Q(T(\mathcal{D}_0))|\mathcal{D}_0\in V_1] &= \int_{q}q\mathbb{P}(Q(T(\mathcal{D}_0))=q|\mathcal{D}_0\in V_1) \\
&= \int_{q}q\mathbb{P}(Q(\mathcal{D}_0)=q|\mathcal{D}_0\in V_2) \\
&= \mathbb{E}[Q(\mathcal{D}_0)|\mathcal{D}_0\in V_2] \\
&= \mathbb{E}[Q].
\end{align*}

Let us finally consider $\pi(Q)$ defined by $\forall \mathcal{D}_0\in V_2, \pi(Q)(\mathcal{D}_0)=Q(\mathcal{D}_0)$ and $\forall \mathcal{D}_0\in V_1, \pi(Q)(\mathcal{D}_0)=Q(T(\mathcal{D}_0))$.

We have 
\begin{align*}
\mathbb{E}[\pi(Q)]&=\mathbb{P}(\mathcal{D}_0\in V_1)\mathbb{E}[\pi(Q)|\mathcal{D}_0\in V_1]+\mathbb{P}(\mathcal{D}_0\in V_2)\mathbb{E}[\pi(Q)|\mathcal{D}_0\in V_2]\\
&=\mathbb{P}(\mathcal{D}_0\in V_1)\mathbb{E}[Q(T(\mathcal{D}_0))|\mathcal{D}_0\in V_1]+\mathbb{P}(\mathcal{D}_0\in V_2)\mathbb{E}[Q|\mathcal{D}_0\in V_2] \\
&= (\mathbb{P}(\mathcal{D}_0\in V_1) + \mathbb{P}(\mathcal{D}_0\in V_2))\mathbb{E}[Q]\\
&= \mathbb{E}[Q].
\end{align*}

Similarly
\begin{align*}
    S(\pi(Q))&=\mathbb{P}(\mathcal{D}_0\in V_1)S(\pi(Q)|V_1)+\mathbb{P}(\mathcal{D}_0\in V_2)S(\pi(Q)|V_2) \\
    &= \mathbb{P}(\mathcal{D}_0\in V_1)S(Q|V_2)+\mathbb{P}(\mathcal{D}_0\in V_2)S(Q|V_2) \\
    &= S(Q|V_2) \\
    &\leq S(Q).
\end{align*}

Finally, $\pi(Q)$ only take values in $V_2\subseteq V^-\cup V_2^+$, thus, $d(\pi(Q))\leq d^-(Q)+\frac 1 2 d^+(Q)\leq 0.75 d(Q)$, as the diameter of $V_2^+$ is half that of $V^+$ by construction. If $S(Q|V_2)\geq S(Q|V_1)$, we can follow the same process except that we project $V_2$ into $V_1$. If $d^-(Q)\geq d^+(Q)$, we can follow the same process except that we divide $V^-$ around its center instead of $V^+$.

By repeating the process we find a sequence of $Q_i:=\pi^i(Q)$ for which $S(Q_i)$ is at most $S(Q)$, the expectation over $dom(\mathcal{D}_0)$ is $\mathbb{E}[Q]$, and both $Q_i(V^-(Q_i))$ and $Q_i(V^+(Q_i))$ shrink to singletons $\{v\}$ and $\{u\}$. 

Then, by introducing 
$m=\frac{1-\alpha'-v}{u-v}$,
$$Q'(\mathcal{D}_0)=Q_{u,v}:=v+(u-v)\mathbbm{1}_{\mathcal{D}_0\in F^{-1}_{dom(\mathcal{D}_0)}(m)},$$
where $F^{-1}_{dom(\mathcal{D}_0)}$ is the quantile function of $\mathcal{D}_0$ for the dictionary order, we get that $Q'$ is feasible (it has expectation exactly $1-\alpha'$) and $S(Q')\leq S(Q)$, with $Q'\in \{u,v\}^{dom(\mathcal{D}_0)}$. 

By introducing $u_\star, v_\star=\mathrm{argmin}_{u,v}(S(Q_{u,v}))$ and 
$$Q_\star = Q_{u_\star,v_\star},$$
$Q_\star$ satisfies the constraint and for any feasible $Q_f$, $S(Q_f)\geq S(Q_f')=S(Q_{u(Q_f), v(Q_f)})\geq S(Q_\star)$.

This concludes the proof.
    
\end{proof}

\subsubsection{Analytical expression for the lower bound with split CP}
\label{appendix:analytical_form}

\Cref{prop:close_form_expected_bound} provides a closed form expression to compute the bound in the case of split CP.

Let us start with two intermediate lemmas that provide the expected value of the $\mathrm{Bin}(n, p)$ exceed $k$ when the probability $p$ follows a Beta distribution.

\begin{lemma}\label{lm:expectation_no_noise_constraint}
    $$
    \mathbb{E}_{p \sim \mathrm{Beta}(r, s)}\bigg[ F_k(p) \bigg] = \frac{1}{B(r,s)} \sum_{i=k}^n\binom{n}{i}B(r+i,s+n-i)
    $$
\end{lemma}

\begin{proof}
\begin{align*}
\mathbb{E}_{p \sim \mathrm{Beta}(r,s)}\bigg[ F_k(p) \bigg] 
&= \int_{0}^1  F_k(p) \frac{p^{r-1}(1-p)^{s-1}}{B(r,s)} dp\\
        &= \frac{1}{B(r,s)}\int_{0}^1 p^{r-1}(1-p)^{s-1}\sum_{i=k}^n\binom{n}{i}p^i(1-p)^{n-i}d p\\
        &= \frac{1}{B(r,s)} \sum_{i=k}^n\binom{n}{i}\int_{0}^1 p^{r-1 +i}(1-p)^{s-1+n-i} dp
\end{align*}

Hence the result, since $B(r,s) = \int_{0}^1 p^{r-1}(1-p)^{s-1} dp$.
\end{proof}

\begin{lemma}\label{lm:expectation_noise_constraint}
\begin{align*}
\mathbb{E}_{p \sim \mathrm{Beta}(r,s)}\bigg[ F_k(p') \mathbbm{1}_{p' \in (0, 1)} \bigg] &= \frac{1}{B(r,s) b^n} \sum_{i=k}^n\binom{n}{i}\sum_{j=0}^i\binom{i}{j}(b-1)^{i-j}  B_{ij} \\
\text{ where } B_{ij} &= \bigg[B(r+j,s+n-i)-B_{\rm{inc}}(1-b,r+j,s+n-i)\bigg]
\end{align*}

\begin{proof}
We denote $p' = 1 - \frac{1 -p}{b}$. As such $p' \in (0, 1)$ if and only if $p \in (1, 1-b)$ and we have
\begin{align*}
\mathbb{E}_{p \sim \mathrm{Beta}(r,s)}\bigg[ F_k(p') \mathbbm{1}_{p' \in (0, 1)} \bigg]  &= \frac{1}{B(r,s)}\int_{1-b}^1 p^{r-1}(1-p)^{s-1}\sum_{i=k}^n\binom{n}{i}\left(1-\frac{1-p}{b}\right)^i\left(\frac{1-p}{b}\right)^{n-i}d p\\
    &= \frac{1}{B(r,s)b^n} \sum_{i=k}^n\binom{n}{i}\sum_{j=0}^i\binom{i}{j}(b-1)^{i-j}\int_{1-b}^1 p^{r-1+j}(1-p)^{n-i+s-1}d p\\
    &= \frac{1}{B(r,s) b^n} \sum_{i=k}^n\binom{n}{i}\sum_{j=0}^i\binom{i}{j}(b-1)^{i-j} B_{i, j}
\end{align*}
where in the second equality, we use $\displaystyle (b -1+p)^i=\sum_{j=0}^{i}\binom{i}{j}p^j(b-1)^{i-j}$
and denoted
$$
B_{ij} = \int_{1-b}^1 p^{r-1+j}(1-p)^{n-i+s-1}d p =  \bigg[B(r+j,s+n-i)-B_{\rm{inc}}(1-b,r+j,s+n-i)\bigg] .
$$
\end{proof}
\end{lemma}

\begin{restatable}{proposition}{weaksplit}
\label{prop:close_form_expected_bound}
Let $B$ be the Beta function and $B_{\rm{inc}}$ the incomplete version
$$B_{\rm{inc}}(z,r,s)=\int_{0}^z x^{r-1}(1-x)^{s-1}dx. $$
% Then $\mathbb{P}(\theta^*\in\Theta_{k})\geq H(k)$,

Without \Cref{assumption:strong}, we have
\begin{align*}
H(k) &= \frac{1}{B(i_\alpha,j_\alpha)b^n} \sum_{i=k}^n\binom{n}{i}\sum_{j=0}^i\binom{i}{j}(b-1)^{i-j}B_{ij}\\
B_{ij} &= B(i_\alpha+j,n-i+j_\alpha) 
- B_{\rm{inc}}(1-b, i_\alpha+j,n-i+j_\alpha),
\end{align*}
and with \Cref{assumption:strong},
\begin{align*}
H(k) &= \frac{1}{B(i_\alpha,j_\alpha)}\sum_{i=k}^n\binom{n}{i}
B(i_\alpha+i,n-i+j_\alpha)
% B_{i}\\
% B_{i} &= B(i_\alpha+i,n-i+j_\alpha).
\end{align*}
\end{restatable}

% \weaksplit*

\begin{proof}

    We remind the notation $Q' = 1 - \frac{1 - Q}{b}$. Then
    from \Cref{eq:iterated_proba}
    \begin{align*}
        \mathbb{P}_{\mathcal{D}_0,\mathcal{D}_1}(\theta_\star\in\Theta_k) = \mathbb{E}_{\mathcal{D}_0}\bigg[F_k(Q_{\star}(\mathcal{D}_0)\bigg] \geq
        \mathbb{E}_{Q \sim \mathrm{Beta}(i_\alpha, j_\alpha)}\bigg[ F_k(Q') \mathbbm{1}_{Q' \in (0, 1)} \bigg]
    \end{align*}
    and we conlude using \Cref{lm:expectation_noise_constraint}

    If we additionally assume \Cref{assumption:strong}, then from \Cref{eq:iterated_proba}
    \begin{align*}
        \mathbb{P}_{\mathcal{D}_0,\mathcal{D}_1}(\theta_\star\in\Theta_k)  = \mathbb{E}_{\mathcal{D}_0}\bigg[F_k(Q_{\star}(\mathcal{D}_0)\bigg] \geq \mathbb{E}_{Q \sim \mathrm{Beta}(i_\alpha, j_\alpha)}\bigg[ F_k(Q) \bigg]
        % = \frac{\sum_{i=k}^n\binom{n}{i}B(i_\alpha+i,j_\alpha+n-i)}{B(i_\alpha,j_\alpha)}
    \end{align*}
    % \begin{align*}
    %     \mathbb{P}_{\mathcal{D}_0,\mathcal{D}_1}(\theta_\star\in\Theta_k) \geq \mathbb{E}_{\mathcal{D}_0}\bigg[\sum_{i=k}^n\binom{n}{i}Q_Y^i\left(1-Q_Y\right)^{n-i}\bigg]
    %     % &= \frac{1}{B(i_\alpha,j_\alpha)}\int_{Q_Y=0}^1 Q_Y^{i_\alpha-1}(1-Q_Y)^{n_{cal}-i_\alpha}\sum_{i=k}^n\binom{n}{i}Q_Y^i(1-Q_Y)^{n-i}dQ_{Y}\\
    %     % &= \frac{\sum_{i=k}^n\binom{n}{i}\int_{Q_Y=0}^1 Q_Y^{i_\alpha+i-1}(1-Q_Y)^{n_{cal}-i_\alpha+n-i}}{B(i_\alpha,j_\alpha)}\\
    %     = \frac{\sum_{i=k}^n\binom{n}{i}B(i_\alpha+i,j_\alpha+n-i)}{B(i_\alpha,j_\alpha)}
    % \end{align*}
    and we conclude using \Cref{lm:expectation_no_noise_constraint}.
\end{proof}

\subsection{High Probability Lower Bounds}
\label{appendix:pac}

Valid coverage does not always holds conditional on the observed data, so we might settle for a high probability guarantee rather than an expected one. Notably, the lower bound in \Cref{prop:noise-free-cov} also holds conditionally on $\mathcal{D}_{0}$.
Thus, $\forall \alpha, \delta \in (0, 1)$, a Probably Approximately Correct guarantees on the coverage of the noisy outputs \citep{vovk2012conditional} translate into the noiseless output as well.

We illustrate in \Cref{fig:pac_deltas}, how the coverage changes as a function of $\delta$.

\begin{figure}[ht]
  \centering
  \includegraphics[width=0.8\textwidth]{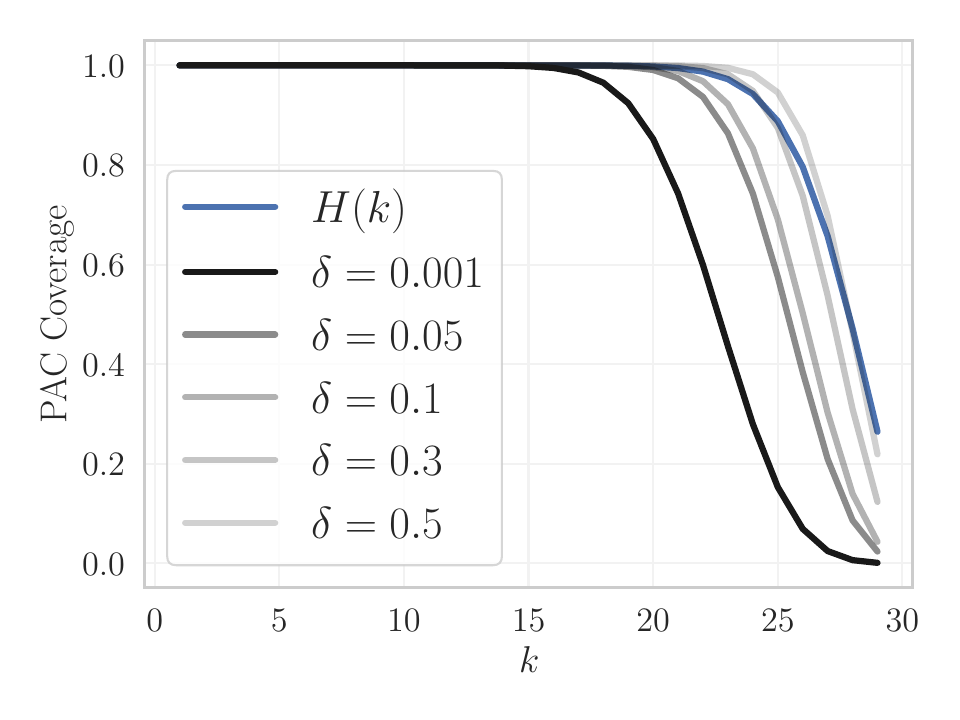}
  \caption{Guaranteed coverage with the PAC bounds when the tolerance level $\delta$ varies. We fix $n=30$ and $n_{\mathrm{cal}} = 50$. The function $H(k)$ corresponds to the lower bound on the expected coverage.}
  \label{fig:pac_deltas}
\end{figure}

\begin{lemma} \label{lm:conditional_cov}
Under the model in \Cref{eq:model}, It holds
$$
\mathbb{P}_{X, Y}\left(f_{\theta_\star}(X) \in \Gamma(X) \mid \mathcal{D}_0\right) \geq
1 - \frac{\mathbb{P}_{X, Y}(Y \notin \Gamma(X) \mid \mathcal{D}_0)}{b} 
$$
\end{lemma}

\begin{proof}
We use same steps than the proof of \Cref{prop:noise-free-cov}.

In this case, we replace the events $E = \{Y \in \Gamma(X)\}$ by $E = \{Y \in \Gamma(X) \mid \mathcal{D}_0\}$.
\end{proof}

\begin{restatable}{proposition}{paccoverage}
\label{prop:paccoverage}
For any $\alpha, \delta \in (0, 1)$,
let us define $\tilde \alpha(\delta) = \alpha + \sqrt{ \frac{\log(1/\delta) }{n_{\rm cal}}}$. Then, it holds
% $$
% \mathbb{P}_{\mathcal{D}_0} \bigg( \mathbb{P}_{X, Y}( f_{\theta_\star}(X) \in \Gamma(X) \mid \mathcal{D}_0 ) \geq 1 - \tilde \alpha(\delta) \bigg) \geq 1 - \delta
% $$
$$
\mathbb{P}_{\mathcal{D}_0} \bigg( Q_\star(\mathcal{D}_0) \geq 1 - \frac{\tilde \alpha(\delta)}{b} \bigg) \geq 1 - \delta,
$$
Furthermore, we also have
% $$
% \mathbb{P}_{\mathcal{D}_0} \bigg( \mathbb{P}_{x, y}( y \in \Gamma(x) \mid \mathcal{D}_0 ) \geq 1 - \alpha(\delta) \bigg) \geq 1 - \delta
% $$
% 
$$
\mathbb{P}_{\mathcal{D}_0} \bigg( \mathbb{P}_{\mathcal{D}}( \theta_\star \in \Theta_{k_{\rm{PAC}}} \mid \mathcal{D}_0 ) \geq 1 - \beta \bigg) \geq 1 - \delta
$$
where
$$
 k_{\rm{PAC}} = \max\left\{k: F_k\left(1 - \frac{\tilde \alpha(\delta)}{b}\right) \geq 1 - \beta \right\}.
$$
\end{restatable}

\begin{proof}
From \Cref{lm:conditional_cov}, we have
$$
Q(\mathcal{D}_0) \geq 1-\tilde \alpha(\delta) \text{ implies }
Q_{\star}(\mathcal{D}_0) \geq 1 - \frac{\tilde \alpha(\delta)}{b}
$$
And from \cite{vovk2012conditional}, it holds
$$
\mathbb{P}_{\mathcal{D}_0}\big(Q(\mathcal{D}_0)  \geq 1 - \tilde \alpha(\delta) \big) \geq 1 - \delta,
$$
Which proves the first inequality.
% \end{proof}

% \pacvoting*
% \begin{proof}
% We remind that the confidence set is
% $$
% \Theta_{k} = \left\{\theta \in \mathbb{R}^d : \sum_{i=1}^{n} \mathbbm{1}_{\theta \in \Theta(X_i)} > k\right\}
% $$
% $$
% \Theta_{Q\left(\frac{\tilde \alpha(\delta)}{b}\right)} = \left\{\theta \in \mathbb{R}^d : \sum_{i=1}^{n_1} \mathbbm{1}_{\theta \in \Theta(X_i)} > Q\left(\frac{\tilde \alpha(\delta)}{b}\right)\right\}
% $$
Since the events $\{\theta_\star \in \Theta(X_i) \mid \mathcal{D}_0\}$ are independent and occurs with probability $Q_\star(\mathcal{D}_0)$, we have
\begin{align*} 
\mathbb{P}(\theta_\star \in \Theta_{k_{\rm PAC}} \mid \mathcal{D}_0) &=
\mathbb{P}\left( \sum_{i=1}^{n} \mathbbm{1}_{\theta \in \Theta(X_i)} \geq k_{\rm PAC} \mid \mathcal{D}_0 \right) \\
&= F_{k_{\rm PAC}}(Q_\star(\mathcal{D}_0)) \\
\end{align*}
And since $Q_\star(\mathcal{D}_0)\geq 1 - \frac{\tilde \alpha(\delta)}{b}$ with probability $1-\delta$, we have with probability at least $1-\delta$ over $\mathcal{D}_0$
\begin{align*}
    \mathbb{P}(\theta_\star \in \Theta_{k_{\rm PAC}} \mid \mathcal{D}_0)&\geq  F_{k_{\rm PAC}}\left(1 - \frac{\tilde \alpha(\delta)}{b}\right) \\
    &\geq 1-\beta,
\end{align*}
by construction of $k_{\rm PAC}$, which concludes the proof.
\end{proof}

\begin{remark}
From \Cref{eq:split-dep}, one also obtains the more accurate PAC lower bound with
$$ 
\tilde \alpha(\delta) = \sup\left\{\gamma : F_{\mathrm{Beta}(i_\alpha, j_\alpha)}\left(1 - \frac{\gamma}{b}\right) \leq \delta \right\}
$$
where $F_{\mathrm{Beta}}$ is the c.d.f. of $\mathrm{Beta}(i_\alpha, j_\alpha)$ distribution.
\end{remark}

\subsubsection{Conformal abstention}
\boundabstention*

\begin{proof}
By definition, we have
$y \text{ and } f_{\theta_\star}(X) \in \Gamma(X)$ if and only if $A(X) \leq Y \leq B(X)$ and $A(X) \leq f_{\theta_\star}(X) \leq B(X)$. Thus,
\begin{align*}
\mathbb{P}(|Y - f_{\theta_\star}(X)| \leq B(X) - A(X)) &\geq \mathbb{P}(Y \text{ and } f_{\theta_\star}(X) \in \Gamma(X)) \\
&\geq 1 - \mathbb{P}(Y \notin \Gamma(X)) - 
\mathbb{P}(f_{\theta_\star}(X) \notin \Gamma(X)) \\
&\geq 1 - \alpha - \frac{\alpha}{b} \\
&= 1 - \left(1 + \frac1b\right) \alpha.
\end{align*}
\end{proof}

\subsection{MILP representation}
\milprepresentation*

\begin{proof}
By definition,
$$\theta\in\Theta_k \Longleftrightarrow \sum_i \phi_i(\theta)\geq k \text{ for } \phi_i(\theta) \in \{0,1\}$$
$\text{for which } \phi_i(\theta)=1 \text{ implies } A(X_i)\leq\theta^\top X_i\leq B(X_i).$
To capture that such constraint should only be active when $\phi_i(\theta)=1$, we use the big M method \citep{wolsey2014integer,hillier2001introduction}, that sets a large constant $M$ and applying the constraint ($\phi_{i}^{c} = 1 - \phi_i$)
$$A(X_i)- \phi_{i}^{c}(\theta) M\leq \theta^\top X_i \leq B(X_i)+\phi_{i}^{c}(\theta) M.$$
One verifies that all the constraints for which $\phi_i(\theta)=0$ are inactive, and increase $M$ otherwise, to ensure the equivalence of the feasible set. 
\end{proof}

\section{Other}
\label{appendix:other}

\subsection{Counter-example to \Cref{assumption:strong}}
\label{appendix:counterexample}
While \Cref{assumption:strong} is very likely to hold, it cannot be theoretically guaranteed, due to adversarial counter-examples.

For instance, consider the following simple 1D counter-example: 
$$\theta^\star=0, \text{ and } \hat{\theta}=1, $$ 
$$\Gamma(X)=[\hat{\theta}X-0.99, \hat{\theta}X+0.99],$$ 

$$\mathbb{P}(\epsilon=0.02) = \mathbb{P}(\epsilon=-0.02)=0.5$$

and
$$\mathbb{P}(X=0)=0.9, \mathbb{P}(X=1)=0.1 $$

In this case, when $X=0$, $\Gamma(X)$ covers both the noisy and noise-free output with probability 1. When $X=1$, $\Gamma(X)=[0.01, 1.99]$ never contains the noise-free output $0$, but contains the noisy output with probability 0.5 (when $\epsilon=0.02$). As such, the coverage of $\Gamma(X)$ for the noisy output is 0.95, while its coverage for the noise free output is 0.9.

\end{document}